%% file: main.tex
\documentclass[twoside]{article}

\usepackage[accepted]{styles/aistats2026/aistats2026}
\usepackage[authoryear,round]{natbib}

\PassOptionsToPackage{usenames,dvipsnames,svgnames,x11names,table}{xcolor}
\usepackage{xcolor}
\usepackage{booktabs}       %

\usepackage{lipsum}         %
\usepackage{graphicx}
\usepackage{subcaption}
\usepackage{thm-restate} %
\usepackage{enumitem} %
\usepackage{wrapfig}
\usepackage{placeins}
\usepackage{packages/colab} %
\usepackage{packages/math} %
\usepackage{packages/refs} %
\usepackage{packages/algo} %
\usepackage{packages/mylistings} %

\usepackage{hyperref}
\usepackage{cleveref}       %
\usepackage{etoc}           %

\usepackage{xcolor}   %
\definecolor{darkblue}{rgb}{0, 0, 0.5}
\hypersetup{colorlinks=true, citecolor=darkblue, linkcolor=darkblue, urlcolor=darkblue}

\usepackage[most]{tcolorbox}
\tcbset{
  boxrule=0.7pt,
  arc=1.5pt,
  left=0.5pt,
  right=0.5pt,
  top=0.5pt,
  bottom=0.5pt,
  boxsep=4pt,
}

\usepackage{packages/colors} %
\usepackage{packages/symbols} %

\usepackage{url}
\usepackage{etoc}
\begin{document}
\etocdepthtag.toc{mtchapter}              %
\etocsettagdepth{mtchapter}{subsection}   %
\etocsettagdepth{mtappendix}{none}        %

\runningtitle{CTMC Framework for Insertion Language Models}

\runningauthor{Patel, Rozonoyer, Das, Naseem, Rudner, McCallum}

\twocolumn[

\aistatstitle{A Continuous-Time Markov Chain Framework\\for Insertion Language Models}

\aistatsauthor{}

\begin{center}
\vspace*{-30pt}
\setlength{\tabcolsep}{4pt}%
\begin{tabular*}{0.93\textwidth}{@{\extracolsep{\fill}}ccc@{}}
\textbf{Dhruvesh Patel} &
\textbf{Benjamin Rozonoyer} &
\textbf{Soumitra Das} \\\vspace*{-8pt}
UMass Amherst &
UMass Amherst &
UMass Amherst \\
\noalign{\vskip 0.25in}
\textbf{Tahira Naseem} &
\textbf{Tim G.~J.~Rudner} &
\textbf{Andrew McCallum} \\
IBM Research &
University of Toronto \& Vijil &
UMass Amherst \\
\end{tabular*}
\end{center}
\vskip 0.3in plus 2fil minus 0.1in
]

\input{drafts/aistats/sections/00_abstract}

\input{drafts/aistats/sections/01_introduction}
\input{drafts/aistats/sections/03_background}
\input{drafts/aistats/sections/04_method}
\input{drafts/aistats/sections/02_related_work}
\input{drafts/aistats/sections/05_experiments}

\FloatBarrier
\subsubsection*{Acknowledgments}
We thank Michael Boratko and Tomas Geffner for helpful discussions and encouragement.
This work was supported by IBM under IBM Research Collaboration Agreement No W1668553, and by National Science Foundation under the grant IIS-2106391.
The views and conclusions contained herein are those of the authors and should not be interpreted as necessarily representing the official policies or endorsements, either expressed or implied, of IBM or NSF.
\bibliography{references}

\clearpage

\section*{Checklist}

\begin{enumerate}

  \item For all models and algorithms presented, check if you include:
  \begin{enumerate}
    \item A clear description of the mathematical setting, assumptions, algorithm, and/or model. [Yes]
    \item An analysis of the properties and complexity (time, space, sample size) of any algorithm. [Not Applicable]
    \item (Optional) Anonymized source code, with specification of all dependencies, including external libraries. [No]
  \end{enumerate}

  \item For any theoretical claim, check if you include:
  \begin{enumerate}
    \item Statements of the full set of assumptions of all theoretical results. [Yes]
    \item Complete proofs of all theoretical results. [Yes, we provide proofs for the new results and provide references for known results.]
    \item Clear explanations of any assumptions. [Yes]     
  \end{enumerate}

  \item For all figures and tables that present empirical results, check if you include:
  \begin{enumerate}
    \item The code, data, and instructions needed to reproduce the main experimental results (either in the supplemental material or as a URL). [Yes, the code used for this work is available at \texttt{\url{https://github.com/dhruvdcoder/ctmc_dilm}}. Appendix~\ref{app:implementation_details} contains some key implementation details, which are also available in the code.] 
    \item All the training details (e.g., data splits, hyperparameters, how they were chosen). [Yes, these are available in the code and in Appendix~\ref{app:implementation_details}.]
    \item A clear definition of the specific measure or statistics and error bars (e.g., with respect to the random seed after running experiments multiple times). [Not Applicable]
    \item A description of the computing infrastructure used. (e.g., type of GPUs, internal cluster, or cloud provider). [Yes, the details are available in Appendix~\ref{app:implementation_details}.]
  \end{enumerate}

  \item If you are using existing assets (e.g., code, data, models) or curating/releasing new assets, check if you include:
  \begin{enumerate}
    \item Citations of the creator If your work uses existing assets. [Yes]
    \item The license information of the assets, if applicable. [Not Applicable]
    \item New assets either in the supplemental material or as a URL, if applicable. [Yes, we provide the URL to the code repository, which also contains URLs for the model checkpoints.]
    \item Information about consent from data providers/curators. [Not Applicable]
    \item Discussion of sensible content if applicable, e.g., personally identifiable information or offensive content. [Not Applicable]
  \end{enumerate}

  \item If you used crowdsourcing or conducted research with human subjects, check if you include:
  \begin{enumerate}
    \item The full text of instructions given to participants and screenshots. [Not Applicable]
    \item Descriptions of potential participant risks, with links to Institutional Review Board (IRB) approvals if applicable. [Not Applicable]
    \item The estimated hourly wage paid to participants and the total amount spent on participant compensation. [Not Applicable]
  \end{enumerate}

\end{enumerate}

\clearpage
\input{drafts/aistats/sections/z0_appendix}

\end{document}

%% file: drafts/aistats/sections/00_abstract.tex
\begin{abstract}
Insertion Language Models (ILMs) offer several advantages over left-to-right generation and mask-based generation. 
However, existing formulations of insertion-based generation have largely been ad-hoc. 
In this paper, we derive a diffusion-style denoising objective for ILMs from first principles by formulating the noising process as a continuous-time Markov chain on the space of variable-length sequences. We show that previous formulations of ILMs can be viewed as special cases of this denoising framework. Through empirical evaluation on a synthetic planning task, we show that the proposed approach retains the benefits of insertion-based generation over left-to-right generation and masked diffusion models. In language modeling, our diffusion-based approach is competitive with left-to-right generation and masked diffusion models, while offering additional flexibility in sampling compared to existing insertion language models.
\end{abstract}

%% file: drafts/aistats/sections/01_introduction.tex
\section{\MakeUppercase{Introduction}}
\label{sec:introduction}
\begin{figure}[!ht]
    \vspace{15pt}
    \centering
    \includegraphics[trim=0 0 0mm 0,clip,width=0.9\linewidth]{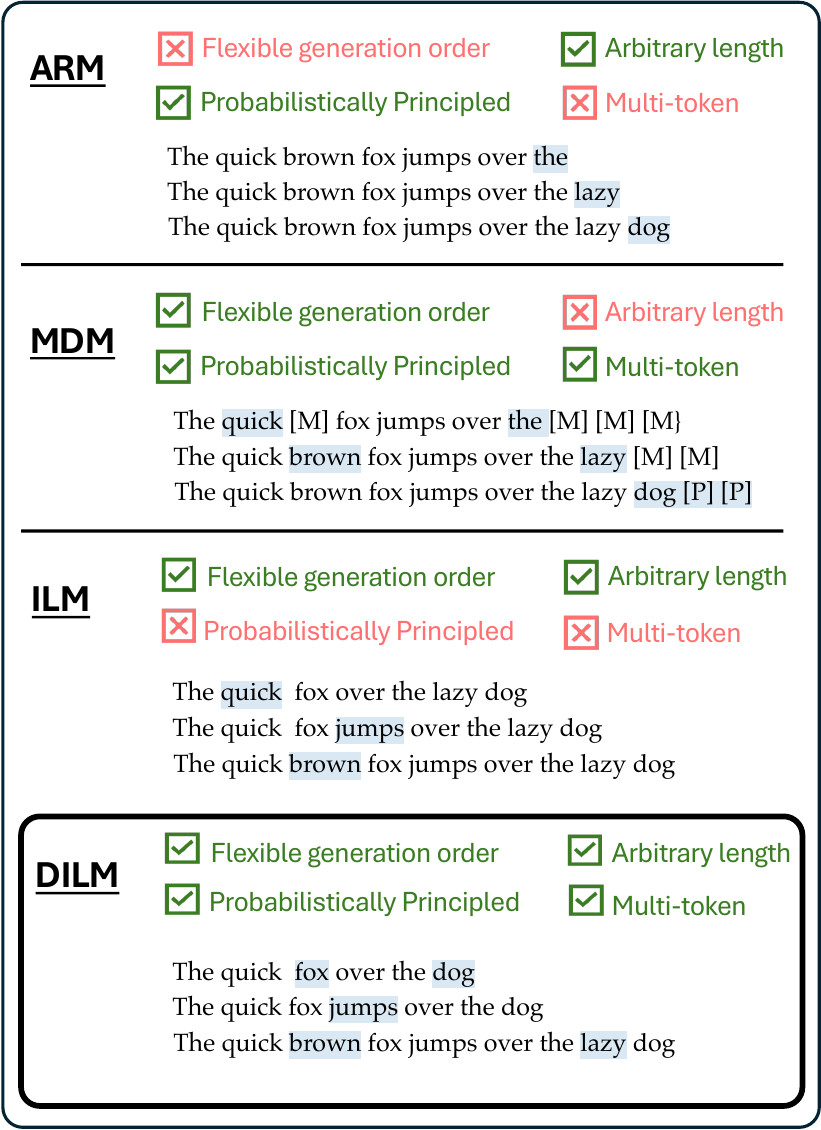}
    \caption{Diffusion-based Insertion Language Models (DILMs) retain the benefits of Insertion Language Models (ILMs) while offering principled training and sampling procedures based on diffusion using a continuous-time Markov chain framework.}
    \label{fig:model_comparison}
\end{figure}
Autoregressive language models (ARMs) generate sequences by predicting tokens from left-to-right. They can naturally generate variable-length sequences by using an \texttt{EOS} token to mark the end of a sequence.
However, certain tasks like planning and sub-goal based infilling require a more flexible approach for sequence generation~\citep{pmlr-v235-bachmann24a}.
Masked diffusion models can generate sequences in arbitrary order while also generating multiple tokens per step~\citep{loudiscrete,sahooSimpleEffectiveMasked2024,shi2024simplified}.
MDMs do not explicitly model the length of the sequence, but are trained to predict variable-length sequences by predicting the \texttt{PAD} tokens like ordinary tokens.
This impacts the quality of generation and restricts the sampling to block-based left-to-right sampling \citep{nie2025scaling,yang2026rhotexttteostrainingfreebidirectionalvariablelength,li2025fixedtrainingfreevariablelengthdenoising,wu2026dreamondiffusionlanguagemodels}.
Insertion Language Models (ILMs) \citep{sternInsertionTransformerFlexible2019a, patel2025insertion}, on the other hand, generate sequences through expansion by iteratively inserting tokens at arbitrary positions.
ILMs combine the benefits of ARMs and MDMs --- they naturally generate variable-length sequences while maintaining the ability to generate in non-left-to-right order and using relative position constraints.
ILMs perform better than MDMs and left-to-right autoregressive models (ARMs) on sub-goal based planning tasks and infilling tasks that require respecting relative ordering constraints \citep{patel2025insertion}.

Existing formulations of ILMs use ad-hoc training objectives and sampling procedures \citep{patel2025insertion,sternInsertionTransformerFlexible2019a}.
In this work, we derive a principled diffusion-style denoising objective and sampling procedures for insertion-based generation.
We formulate the noising process as a continuous-time Markov chain over the space of sequences that drops tokens uniformly with a time-dependent rate.
We present two noising processes and respective transformer-based parameterizations of the rate matrix of the generative process (the reverse of the noising process).
The first parameterization models the joint probability of vocabulary item and insertion location, and the probability of the length-to-go.
The second parameterization models the rate matrix of the generative process using independent, per-dimension probabilities allowing one to sample multiple insertions at once.
We derive the evidence lower bound on the data log-likelihood under the parameterized generative process, and discuss the necessary computational considerations for training and sampling.

The main contributions of this work are as follows:\vspace*{-3pt}
\begin{enumerate}[topsep=0pt, align=left, leftmargin=15pt, labelindent=1pt,
listparindent=\parindent, labelwidth=0pt, itemindent=!, itemsep=3pt, parsep=0pt]
    \item We present a principled diffusion-style denoising framework for insertion-based generation and propose two variants of the generative process in this framework. \Cref{fig:model_comparison} shows how the proposed generative model relates to existing sequence generation models.
    \item We show that the resulting sampling procedures unify the existing formulations of Insertion Language Models.
    \item Empirical evaluation on synthetic planning tasks and language modeling demonstrates improved performance over existing ILMs and MDMs.
\end{enumerate}
\vspace*{5pt}
\begin{tcolorbox}[
    colback=blue!5!white,
    colframe=blue!70!black,
  ]
    Code to reproduce our experiments is available at:\\[-15pt]
    \begin{center}
        \href{https://github.com/dhruvdcoder/ctmc_dilm}{\texttt{https://github.com/dhruvdcoder/ctmc\_dilm}}
    \end{center}
\end{tcolorbox}

\newpage

\begin{figure*}
    \centering
    \includegraphics[trim=0 2mm 0mm 0mm,clip,width=0.9\textwidth]{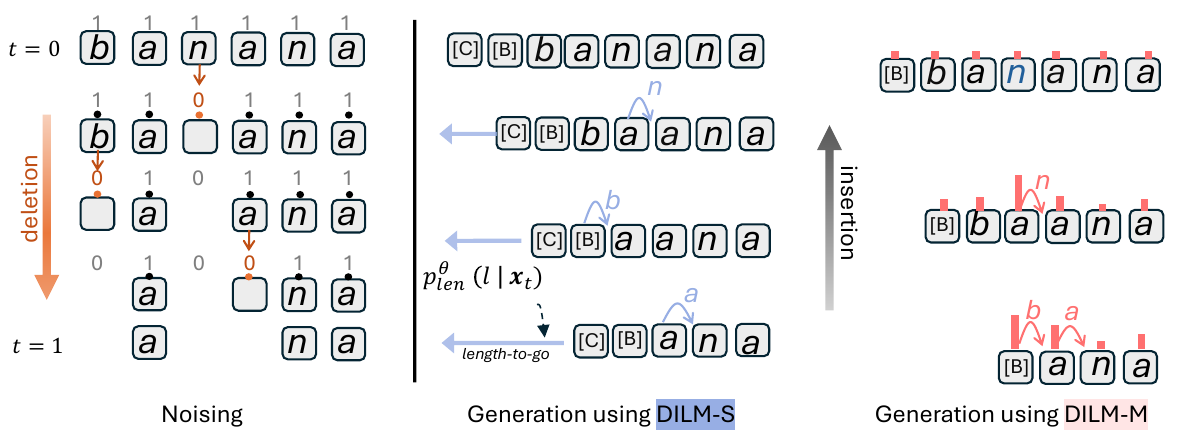}
    \vspace{-5pt}
    \caption{The noising process, shown on the left, is a continuous-time Markov chain that deletes one token at a time uniformly with the deletion events arriving with rate $\sigma_t$.
    The {\color{gray}bit vectors} show the {\color{powerpoint_orange_red_main}{deletion}} path w.r.t.\ the original sequence $\x_0 = (6, \textnormal{banana})$. 
    On the right we show two parameterizations of the generative process. 
    DILM-S predicts the length-to-go and a joint probability $p^\theta_{\ins}(i, w\mid \x_t)$ over insertion locations and vocabulary items.
    The generation stops when the predicted insertion rate modeled through the predicted \emph{length-to-go} value, shown using the blue arrows pointing left, goes to zero.
    DILM-M uses per-position insertion rate, shown as the vertical blue bars, and $p^\theta_{\ins}(w \mid \x_t, i)$. The generation stops when the insertion rate is low for all positions. \texttt{[C]} and \texttt{[B]} are special tokens, where the former is used by \idlm{} to perform length-to-go prediction and the latter is used by both models as a beginning-of-sequence marker to predict the token at the first gap.
    }
    \label{fig:deletion_path}
    \vspace{-8pt}
\end{figure*}

%% file: drafts/aistats/sections/03_background.tex
\section{\MakeUppercase{Background}}
\label{sec:background}
In this section, we will state some elementary results from the theory of Markov chains to set up the framework for our method.
We will also introduce some notation that will be used throughout the paper (see \Cref{app:notation} for a summary of all notation).

\xhdr{Notation.}
Capital letters (e.g. $X, \cB$) are used to denote (scalar or vector valued) random variables and subscripts are used to denote the time index of stochastic processes (e.g., $X_t$, $\cB_t$, etc.).
Boldface vectors and superscripts will be used to denote components of a vector, e.g., $x^i, X_t^i$, for the $i$-th component of $\x$ and $\cX_t$, respectively. %
Double square brackets are used to denote the set of natural numbers up to a specific number, i.e., $\natset{n} = \{1, 2, \ldots, n\}$.
All stochastic processes are assumed to be continuous-time unless an \underline{underline} is used, which denotes a discrete-time process (e.g. $\cX_t, \dX_k$ are continuous- and discrete-time processes, respectively).
We will use the shorthand 
$p_{s|t}(y \mid x)$ to denote the transition probability $\sP(X_s = y \mid X_t = x)$, and in case of discrete time processes, we will use $\ubar p_{r|k}(y \mid x)$.

\xhdr{Continuous-Time Markov Chains.}
Here we briefly discuss continuous-time Markov chains (CTMCs), and their characterization in terms of the rate matrix.
In \cref{sec:method}, we will leverage the relationship between DTMCs and CTMCs to describe our noising process.
Let $X_t$ be a CTMC taking values in a countable state space $\sX$ with right-continuous sample paths.
Its dynamics can be described using a state- and time-dependent jump rate $\lambda_t(x)$ and a post-jump transition kernel $\K[t]{x}{y} \geq 0$ satisfying $\K[t]{x}{x}=0$ and $\sum_{y \in \sX} \K[t]{x}{y} = 1$.
The corresponding transition rate matrix of $X_t$ is then given by
\begingroup
    \setlength{\abovedisplayskip}{4pt}
    \setlength{\belowdisplayskip}{4pt}
\begin{align}
\label{eq:ctmc:transition_rate_matrix}
      \R{t}{x}{y} &= \begin{cases}
          \lambda_t(x) \K[t]{x}{y} & \text{if } x\not=y,\\
          -\lambda_t(x) & \text{if } x = y,
      \end{cases} ~~~\text{with}\nonumber\\
      \K[t]{x}{x} &= 0~~\text{and}~~ \sum_{y \in \sX} \K[t]{x}{y} = 1.
  \end{align}
\endgroup
An informal description of the evolution of the transition probability {in continuous time} for small values of $h$ is given by the first-order discretization
\begingroup
    \setlength{\abovedisplayskip}{4pt}
    \setlength{\belowdisplayskip}{6pt}
\begin{align}
\label{eq:ctmc:first_order_discretization}
    p_{t+h|t}(y \mid x) = \delta_{x}(y) + h\,R_t(x, y) + o(h),
\end{align}
\endgroup
where $\delta_{x}(y)$ is the indicator function that is $1$ if $x=y$ and $0$ otherwise. 
{Intuitively, the tendency $\delta_{x} (y)$ to remain in the same state $x$ is counteracted by the passage of time and the rate parameter $\lambda_t(x) = \sum_{y\not=x} R_t(x,y)$.}

\paragraph{Time Reversal.}
On the time interval $[0, T]$, for any regular CTMC $X_t$ on countable state space, with initial distribution $p_0$ and rate $R_t$, the time reversal $\hat X_{t}\coloneq X_{T-t}$ is also a CTMC with rate $\hat R_t(x, y) = \frac{p_{t}(y)}{p_t(x)} R_{t}(y, x).$\footnote{We have relabeled the time index of the reverse process $T-t$ as $t$ to keep the notation simple. Also, at the jump points the time reversal's value is the left limit of the forward process.}
Let $\hat R^\theta_t$ denote a parameterized rate for the time reversal $\hat X_{t}$ for some parameter value $\theta$. Then the following proposition gives a lower bound on the log-likelihood of the data under the parameterized model.
\begin{restatable}[ELBO]{proposition}{ThmELBO}\label{prop:base_elbo}
    For the time interval $[0,1]$, let $p_0=p_{\textnormal{data}}$ and $p_1=p_{\textnormal{ref}}$ denote the terminal time marginals of the noising process $X_t$. 
    Then for the reverse CTMC with initial distribution $p_{\textnormal{ref}}$, parameterized rate $\hat R_t^\theta$, and the induced terminal marginal $p_0^\theta$, the
     log-likelihood $\E\left[\log p^\theta_0(X_0)\right]$ under the model is bounded from below by
    \begingroup
        \setlength{\abovedisplayskip}{4pt}
        \setlength{\belowdisplayskip}{4pt}
    \begin{equation*}
         \hspace{-2mm}\E \sum_{y \not=x} \left[\hspace{-1pt}-\hat R^\theta_t(x, y) \hspace{-1pt}+\hspace{-1pt} \frac{p_{t|0}(y | x_0)}{p_{t|0}(x | x_0)}R_t(y, x)\log \hat R^\theta_t(x, y)\right] ,%
    \end{equation*}
    \endgroup
    where the expectation is taken with repect to \mbox{$x_0\sim p_{\textnormal{data}}$}, $t\sim \Uniform[0, 1]$, and $x\sim p_{t|0}(x \mid x_0)$.%
\end{restatable}
The result follows by applying Dynkin's formula for the change of measure to the CTMC path measures \citep{hanson2007applied}, followed by the data processing inequality. We provide an intuitive proof using elementary techniques in Appendix~\ref{proof:elbo}.

%% file: drafts/aistats/sections/04_method.tex
\section{\MakeUppercase{Diffusion-Based Insertion Language Models}}
\label{sec:method}

For ease of exposition, we will make the sequence length explicit {by incorporating it into the state-space,} and therefore our domain will be the set of all sequences up to a maximum length $\sX\coloneqq\bigcup_{n=0}^{\lmax} \left(\{n\}\times \vocab^n\right)$, with a single example being the tuple $\x=(n, \seqx)$, where $n$ is the sequence length and $\seqx \in \vocab^n$ the actual sequence of tokens.
We will omit the dot in $\seqx$ and write it as $\x$ when the distinction is not needed.
We will use the shorthand $\X_n$ to mean $\{n\}\times \vocab^n$ {— the collection of all sequences of length $n$}.

\subsection{Noising Process}\label{sec:noise:poisson}
The noising process proceeds by deleting tokens from the {ground-truth} sequence $\x_0$.
To provide intuition, we first formulate the deletions as a homogeneous discrete-time Markov chain with no self-transitions, and then embed it in continuous time.

Formally, any deletion process that does not depend on the content of the sequence can be described by expressing the transitions $(n, \seqx) \rightarrow \cdots \rightarrow (m, \seqy)$ using bit vectors in $\sB_{n}=\bigcup_{m=0}^n \sB_{n, m}$, where $\sB_{n, m}$ is the set of binary vectors of length $n$ with exactly $m$ ones, i.e. $\sB_{n, m} = \{\b \in \{0, 1\}^{n} ~:~ \sum_{i=1}^n b^i = m\}$.
Let $\Idx(\bm b) \coloneqq \{i \in \natset{n} \mid b^i = 1\}$ be the set of indices of the ones in $\bm b$, and $\x[\bm b] \coloneqq \x^{\Idx(\bm b)}$ be the subsequence of $\x$ that only contains elements corresponding to the ones in $\bm b$.
For $\bm b_1, \bm b_2 \in \sB_{n}=\bigcup_{m=0}^n \sB_{n, m}$, we define a partial order $\bm b_1 \succ \bm b_2$ if $b_1^i*b_2^i=b_2^i$ for all $i\in \natset{n}$ and $\sum_{i=1}^n b_1^i > \sum_{i=1}^n b_2^i$, i.e. we call $\bm b_1$ a predecessor of $\bm b_2$ if the positions that are $1$ in $\bm b_2$ are also $1$ in $\bm b_1$, and the number of ones in $\bm b_2$ is less than the number of ones in $\bm b_1$, e.g., $11111 \succ 10101 \succ 10001$.
Using this notation, a general $r$-step deletion path $(n_0, \seqx_0) \rightarrow (n_1, \seqx_1) \rightarrow \dots \rightarrow (n_{r-1}, \seqx_{r-1}) \rightarrow (n_r, \seqx_r)$ can be expressed using $r$ bit vectors $\bm b_1 \succ \dots \succ \bm b_r$ in $\sB_{n_0}=\bigcup_{k=0}^{n_0}\sB_{n_0, k}$ such that $\x_k = \x_0[\bm b_k]$.
To make this concrete, consider $\x_0 = (6, \textnormal{banana})$, and the deletion path specified by the bit vectors $\bm b_1 = 110111, \bm b_2 = 010011$, resulting in $\x_2 = (3, \textnormal{ana})$ as shown in \Cref{fig:deletion_path}.
Notice that there could be multiple deletion paths that lead from $(n, \seqx_0)$ to $(m, \seqx_r)$; in our example, $\b_1 = 011111 \succ \b_2 = 000111$, or just $\b_1=000111$, would also lead to $(3, \textnormal{ana})$.
Since the deletion process does not depend on the content of the state $\x$ but only on the length, the probability of
one step of deletion has the form $\kappa(\bm b_{k+1} \mid \bm b_k)\propto\delta_{\{\bm b_k \succ \cdot\}}(\bm b_{k+1})\, u(\bm b_k, \bm b_{k+1})$, where $u$ is an unnormalized joint probability mass function.
In order to make the process analytically tractable, we further restrict the deletion process such that it deletes \emph{one token at a time} uniformly at random, which implies that $u(\bm b_k, \bm b_{k+1})=\delta_{|\bm b_k| -1}(|\bm b_{k+1}|)$, where $|\bm b_k|=\sum_{j} b_k^j$ is the number of $1$s in $\bm b_k$.
Putting these constraints together we get 
\begingroup
    \setlength{\abovedisplayskip}{4pt}
    \setlength{\belowdisplayskip}{4pt}
\begin{equation*}
\kappa_{\text{Uni}}(\bm b_{k+1} \mid \bm b_k) = \frac{1}{|\bm b_k|} \delta_{\{\bm b_{k} \succ \cdot\}}(\bm b_{k+1})~\delta_{|\bm b_k| -1}(|\bm b_{k+1}|).
\end{equation*}
\endgroup
The following lemma shows that such a noising process is a homogeneous DTMC with transition kernel $\K{\x}{\y}$ that produces a closed-form expression for the $r$-step transition probability $\ubar p_{k+r|k}(\y \mid \x)$.

\begin{restatable}[Deletion DTMC]{lemma}{LemmaDeletionDTMC}\label{lemma:deletion_dtmc}
    Let $\dX_k$ be a discrete-time stochastic process taking values in $\sX$ with one-step transition probabilities governed by the deletion kernel $\kappa_{\textnormal{Uni}}$.
    Then, letting $\hat{C}^n_m \coloneqq \frac{(n-m)!}{n (n-1) \dots (m+1)}$, the process $\dX_k$ is a DTMC with one-step transition kernel and transition probability given by
    \begingroup
        \setlength{\abovedisplayskip}{4pt}
        \setlength{\belowdisplayskip}{4pt}
    \begin{align*}
        &\K{(n, \seqx)}{(m, \seqy)} = \frac{1}{n}\sum_{\bm b \in \sB_{n, m}} \delta_{n-1}(m)\, \delta_{\seqx[\bm b]}(\seqy),\\
        &\ubar p_{k+r|k}\left((m, \seqy)| (n, \seqx)\,\right) = \hat{C}^n_m ~\delta_{r}(n-m)\,\sum_{\bm b \in \sB_{n,m}}\delta_{\seqx[\bm b]}(\seqy).
    \end{align*}\\[-25pt]
    \endgroup
\end{restatable}
\textbf{\emph{Proof sketch.}} To get the intuition for the expression for transition probability, note that for $\kappa_{\textnormal{Uni}}$, the probability of any $(n-m)$ step delete path $\bm b_{n-m} \succ \dots \succ \bm b_1$ is the same and is equal to $\frac{1}{n} \frac{1}{n-1} \dots \frac{1}{m+1}$.
Meanwhile, the number of paths one can take from $\bm b_{n-m}$ down to $\bm b_1$ is ${\binom{n-m}{1}} {\binom{n-m-1}{1}} \dots {\binom{1}{1}} = (n-m)!$.
Therefore, the required probability is $\hat{C}^n_m = \frac{(n-m)!}{n(n-1)\dots(m+1)} = \frac{1}{{\binom{n}{n-m}}}$.
The formal proof is given in Appendix~\ref{app:proofs}. 

\begin{remark}
Note that the transition probability is not uniform over the set of sequences reachable from $(n, \seqx)$ in $r$ steps. Specifically, if there is more than one mask $\bm b \in \sB_{n, n-r}$ such that $\seqx[\bm b] = \seqy$, then the transition probability to land on $\seqy$ is the sum of the probabilities over all such cases.
One can obtain a similar closed-form expression for the transition probability for any deletion kernel that only depends on the length and not the content of the sequence; for example, the right-to-left deletion is a special case, which leads to the usual left-to-right AR generative model.
\end{remark}

Having described the deletion process using discrete steps, we now introduce the jump rate $\lambda_t(\x)$ to complete the description of the continuous-time deletion process (\cref{eq:ctmc:first_order_discretization,eq:ctmc:transition_rate_matrix}).
Intuitively, $\dX_k$ is converted into a continuous-time process $\cX_t$ by dispersing the deletion steps over the time interval $[0, 1]$ with deletion rate $\lambda_t(\x)$ per unit time. 

For efficient sampling from the noising process, we keep the rate $\lambda_t(\x)$ independent of the state $\x$ except for the case when the length of the sequence is $0$, in which case the rate is $0$ as well, i.e., the rate is $\lambda_t((r, \seqx))=\sigma_t\,\delta_{\{\cdot>0\}}(r)$, where $\sigma: [0, 1] \to \sR_+$ is a scalar noise schedule.
The following proposition gives the transition probability under this noising process.
\begin{restatable}[Transition Probability]{proposition}{PropTransitionProbability}\label{prop:transition_probability}
    Let $\cX_t$ be a continuous-time stochastic process described in \eqref{eq:ctmc:transition_rate_matrix}, with transition rate $\lambda_t((r, \seqx))=\sigma_t\,\delta_{\{\cdot>0\}}(r)$ and transitions governed by the transition kernel $K$ that drops one token at a time uniformly at random. Then $\cX_t$ is an inhomogeneous CTMC 
    with transition probability 
    for $s \leq t$ given by
    \begingroup
        \setlength{\abovedisplayskip}{6pt}
\setlength{\belowdisplayskip}{4pt}
    \begin{align*}
    &
    p_{t|s}^{\textnormal{joint}}\left(\,(m, \seqy) \mid (n, \seqx)\,\right)
    \\
    &
    =\begin{cases}
        \frac{ \exp(-\bar \sigma_{s, t}) (\bar \sigma_{s, t})^{n-m}~m!}{n!} \sum\limits_{b\in \sB_{n,m}}\delta_{\seqx[\bm b]}(\seqy), &  0 < m \leq n \\
        1 - \sum\limits_{r=1}^n \frac{ \exp(-\bar \sigma_{s, t}) (\bar \sigma_{s, t})^{n-r}~r!}{n!}, & m = 0
    \end{cases}
    \end{align*}
    \endgroup
    where $\sigma: [0, T] \to \sR_+$ is a scalar noise schedule, $\bar \sigma_{s, t} = \int_s^t \sigma(u) \dd{u}$.
\end{restatable}
\textbf{\emph{Proof sketch.}} We can decouple the number of tokens dropped at time $t$ from the choice of tokens dropped, where the former has a Poisson distribution with mean rate $\bar \sigma_{s, t}$ and the latter has a distribution given in Lemma \ref{lemma:deletion_dtmc}. We separately take care of the case when the number of tokens dropped is equal to the total number of tokens in the sequence. The complete proof is given in Appendix~\ref{app:proofs}.

Proposition \ref{prop:transition_probability} allows us to sample from the noising process by first sampling the number of positions to drop using the Poisson distribution with rate $\bar \sigma_{0, t}$, with maximum drops equal to the length of the sequence, and then sampling the positions to drop by permuting the indices of the sequence and picking the first $n - d$ positions.
Sampling from this noising process, which we call \jointhighlight{Joint Noising}, is described in \Cref{alg:sample_noise}.

\xhdr{\independenthighlight{Independent Noising.}} 
We can also obtain uniform random deletions of positions as the noising process by giving \emph{each position} an \emph{independent} deletion rate $\sigma_t$. 
 In particular, the corresponding transition probability for $s \leq t$ is 
\begingroup
    \setlength{\abovedisplayskip}{4pt}
    \setlength{\belowdisplayskip}{4pt}
\begin{equation*}
    p_{t|s}^{\textnormal{ind}}\bigl((m, \seqy) \mid (n, \seqx)\bigr)
    = \rho_{s,t}^m (1-\rho_{s,t})^{n-m} \sum_{\bm b \in \sB_{n,m}} \delta_{\seqx[\bm b]}(\seqy),
\end{equation*}
\endgroup
where $\rho_{s,t} = \exp(-\bar \sigma_{s,t})$.
We can sample from this noising process by sampling the number of positions to drop using the $\Binomial(n, \rho_{s,t})$ distribution instead of $\Poisson(\bar \sigma_{s,t})$.
The same procedure can be used to sample from the noising process as described earlier, but with the Poisson distribution replaced by the Binomial distribution.
Appendix~\ref{app:independent_deletion} provides additional details and the analogue of Proposition~\ref{prop:transition_probability} for the independent noising process.

\begin{algorithm}
    \caption{Sampling from the Noising Process}
    \label{alg:sample_noise}
    \begin{algorithmic}[1]
        \REQUIRE Sequence $(n, \seqx)$, rate function $\sigma$, $\algvarsc{NoisingType} \in \{{\color{poisson}\algvarsc{Joint}}, {\color{independent}\algvarsc{Independent}}\}$
        \STATE $t \sim \Uniform[0, T]$
        \IF{$\algvarsc{NoisingType} = \color{poisson}\algvarsc{Joint}$}
            \STATE Sample $d \sim \Poisson(\bar \sigma_{0, t})$ \MyCommentShort{\# positions to drop}
        \ELSIF{$\algvarsc{NoisingType} = \color{independent}\algvarsc{Independent}$}
            \STATE Sample $d \sim \Binomial(n, 1-\exp(-\bar \sigma_{0, t}))$ 
        \ENDIF
        \STATE $d \gets \min(d, n)$
        \STATE Permute the indices $(1, \dots, n)$ of $\seqx$ to get $\pi$.
        \STATE Pick the first $n - d$ positions from $\pi$ to get $\tilde \pi \gets \pi([1: n - d])$
        \STATE Prepare mask $\bm b \in \sB_{n, n - d}$ such that $\bm b[i] = 1$ if $i \in \tilde \pi$ and $0$ otherwise.
        \STATE $\seqy \gets \seqx[\bm b]$ \MyCommentShort{Remove the positions to drop}
        \RETURN $(\,(n - d, \seqy),\, t\,)$
    \end{algorithmic}
\end{algorithm}
\vspace{-5pt}
\subsection{Generative Process}\label{sec:poisson:time_reversal}
\vspace{-5pt}

In our case, the time reversal of the noising CTMC is the generative process that constructs a sequence by inserting tokens. 
Let $\sX_n$ denote $\{n\}\times \vocab^n$,~ $\ins: \sX_n\times \natset{n+1}\times \vocab \to  \sX_{n+1}$ denote the insertion operation such that $\ins(\x, i, a)=(n+1,\,x^1, \dots x^{i-1}\,a\,x^i\dots x^n )$, and $\rins(\x) = \{ \ins(\x, i, a) \mid i \in \natset{n+1}, a \in \vocab \}$.

We parameterize the reverse rate matrix directly as $\hat R^\theta_t$.
Next, we discuss different options for parameterizing the rate $\hat R_t^\theta$ of the generative process. We also remark on how these parameterizations relate to some existing formulations of insertion language models.

\subsubsection{Joint Insertion}\label{sec:length_to_go_parameterization}
For the \jointhighlight{joint noising process}, we parameterize the rate using a joint probability distribution over positions and tokens as follows.
Let $\hat p_t^\theta: \sX\times [0,1] \to \Delta^{n \times |\vocab|}$ be the parameterized joint probability distribution $(\x, t) \mapsto \hat p_t^\theta(i, w|\x)$ that takes as input sequence $\x=(n, \seqx)$, time $t$, and an insertion location $i \in \natset{n}$ and produces a joint over the insertion location $i\in \natset{n}$ and vocab item $w$.
Then the generative transition kernel, given that a transition out of $\x$ occurs is $\hat K_{t}^{\theta}(\x, \y) = \sum_{i=1}^{m+1} \sum_{w \in \vocab}  \,\hat p_{t}^{\theta}(i, w \mid \x)\, \delta_{\ins(\x, i, w)}(\y)$, and the fully specified rate of the generative process is
\begingroup
\setlength{\abovedisplayskip}{4pt}
\setlength{\belowdisplayskip}{4pt}
\begin{align}
    \hspace{-3mm}\hat R^\theta_t(\x, \y) = \hat \lambda^\theta_t(\x)\sum_{i=1}^{m+1} \sum_{w \in \vocab}  \,\hat p_{t}^{\theta}(i, w \mid \x)\, \delta_{\ins(\x, i, w)}(\y), \nonumber
\end{align}
\endgroup
where $\hat \lambda_t^\theta$ is the parameterized jump rate.
Plugging this parameterization into the bound in \Cref{prop:base_elbo}
results in a tractable estimate of the ELBO as shown in the following corollary.

\begin{restatable}[ELBO: Joint Insertion]{corollary}{CorELBO}\label{cor:elbo}
    A (biased) estimate of the upper bound on the negative log-likelihood in \Cref{prop:base_elbo} for the joint parameterization of $\hat R_t^\theta$ is given by
    \begin{align*}
        &\mathcal L(\theta) = \mathcal L_{\textnormal{len}}(\theta) + \mathcal L_{\textnormal{token}}(\theta), \text{where} \\
        &\mathcal L^{\textnormal{len}}(\theta) = \E\left[ \hat \lambda^\theta_t(\x_0[\bm b]) - \gamma_t^{\textnormal{joint}} \log \hat \lambda^\theta_t(\x_0[\bm b])\right], \text{and} \\
        &\mathcal L^{\textnormal{token}}(\theta) = -\E\left[ \gamma_t^{\textnormal{joint}} \sum_{i,w} q(i, w)\log \hat p_{t}^{\theta}(i, w\mid \x_0[\bm b])\right].
    \end{align*}
    Here, the expectation is over $\x_0 \sim p_{\textnormal{data}}$, $t\sim \textnormal{Uniform}[0, T]$, $(n-m) \sim \textnormal{Poisson}(\bar \sigma_{0, t})$, $\delta_{\{\cdot \geq 0\}}(m)$, $\bm b \sim \textnormal{Uniform}(\sB_{n, m})$, the inner sum in $\mathcal L^{\textnormal{token}}(\theta)$ is over $i\in \natset{m+1}, w \in \vocab$, and $\gamma_t^{\textnormal{joint}} = \frac{ \sigma_t\,(n-m)}{\bar \sigma_{0, t}\,\tilde S_{n,m,\sigma_t}}$ with $\tilde S_{n,m,\sigma_t} = \left[(1-\delta_0(m)) + \delta_0(m) S_{n,\sigma_t}\right]$,
    $S_{n, \sigma_{t}} = \sum_{k=0}^{\infty} \frac{n!\,\bar \sigma_{0, t}^{k}}{(n+k)!} $,
    $\x_0=(n, \seqx_0)$,  and $q(i, w)=\frac{\sum_{k\in s_i(\bm b)}\delta_{w}(x_0^k)}{n-m}$ is the normalized count of the vocabulary item $w$ occurring in $\x_0$ between the $(i-1)$-th and $i$-th 1s, with $s_i(\bm b)$ being the set of indices of 0s in $\bm b$ that fall between the $(i-1)$-th and $i$-th 1s.
\end{restatable}
\begin{remark}
This estimate is biased because we replace $\delta_{\ins(\x, i, w)}(\y) \log K^\theta_t(\x, \y)$ by $\log \hat p_t(i, w | \x)$, where we have removed the summation present inside the $\log$. This approximation is only active in the case where two insertions lead to the exact same sequence, which only happens in the case of contiguous repeated tokens. For example, if $\x=\textnormal{aaa}$, then $\ins(\x, 1, \textnormal{a})=\ins(\x, 2, \textnormal{a})=\textnormal{aaaa}$. 
\end{remark}

\begin{remark}
    Here we have obtained the estimator by fixing the latent alignment between the sequence $\x_0$ and the subsequence $\x_0[\bm b]$ to be equal to the one given by $\bm b$ itself, but we could marginalize over all $\bm a \in \sB_{n,m}(\x_0, \bm b)$ explicitly. This can be viewed as Rao-Blackwellizing the latent alignment variable $\bm a$, yielding a lower-variance estimator of the ELBO term. However, this requires solving one dynamic programming instance per $(\x_0, \bm b)$ pair to compute all possible alignments of the subsequence $\x_0[\bm b]$ and $\x_0$, which can introduce significant computational overhead during training. We leave an efficient implementation of this approach to future work.
\end{remark}
A detailed discussion and the complete derivation of the corollary is provided in Appendix~\ref{proof:cor_elbo}. A method to efficiently compute $S_{n, \sigma_t}$ on the GPU using \texttt{hyp1f1} is provided in Appendix~\ref{app:compute_ratio}.

\xhdr{Neural Network.} The joint probability distribution over positions and tokens is parameterized using a standard transformer with time conditioning using adaptive layer-norm \citep{peeblesScalableDiffusionModels2023}.
For representing each available gap, we prepend a special \texttt{BOS} (beginning of sequence) token to the sequence, so that $\tilde \x_t=(\texttt{BOS}, x_t^1,\dots, x_t^n)$ has length $n+1$ and each insertion location can be indexed by the token immediately preceding it.
Let $\dec$ denote the transformer with bidirectional attention (without the final linear projection), which takes $\tilde \x_t$ and time $t \in [0, 1]$ as input and returns $\dec(\tilde \x_t; t)\in \sR^{(n+1)\times d}$.
For each insertion location $i\in \natset{n+1}$ the corresponding output of the transformer backbone $\dec(\tilde \x_t; t)^i\in \sR^d$ is passed through the final linear layer $\mlpins: \sR^d \to \sR^{\abs{\vocab}}$ to get a score for the pair $(i, w)$.
\begingroup
    \setlength{\abovedisplayskip}{4pt}
\setlength{\belowdisplayskip}{4pt}
\begin{align*}
    &s^{\theta}(i, w \mid \x_t) = \mlpins\left(\dec(\tilde \x_t; t)^{i}\right)^{w} \\ 
    &\hat p_t^{\theta}(i, w \mid \x_t) = \frac{\exp(s^{\theta}(i, w \mid \x_t))}{\sum_{i'=1}^{n+1}\sum_{w'\in \vocab} \exp(s^{\theta}(i', w' \mid \x_t))}.
\end{align*}
\endgroup
The superscripts $i,w$ above indicate position and token index, respectively.
Parameterizing the insertion rate $\hat \lambda^\theta_t$ as a non-negative scalar is challenging \citep{campbellTransDimensionalGenerativeModeling2023,patel2025insertion}.
However, we can use the length posterior to parameterize the insertion rate as shown in \Cref{prop:length_posterior_rate} in Appendix~\ref{app:proofs}. Specifically, denoting $p_{t,\textnormal{len}}(l \mid \x) \coloneqq \sP(N_0 - N_t = l \mid \cX_t = \x)$, with $\x=(m, \seqx)$, the corresponding insertion rate is given by
\begingroup
    \setlength{\abovedisplayskip}{4pt}
\setlength{\belowdisplayskip}{4pt}
\begin{align*}
    \hat \lambda^\theta_t(\x) =\begin{cases} 
                                    \frac{\sigma_t}{\bar \sigma_{0,t}\,} \E\limits_{l \sim \hat p_{t,\, \text{len}}^{\theta}(\cdot \mid \x)}   l & \text{if } m>0 \\
                                    \frac{\sigma_t}{\bar \sigma_{0,t}\,} \E\limits_{l \sim \hat p_{t,\, \text{len}}^{\theta}(\cdot \mid \x)}   \frac{l}{ S_{l, \sigma_t}} & \text{if } m=0.
    \end{cases}
\end{align*}
\endgroup
We parameterize $\hat p_{t,\text{len}}^{\theta}(l \mid \x_t)$ using a categorical distribution on $\natset{n_{\max}}$ and a dedicated input position in the transformer backbone.
In the implementation, the network outputs logits for the full distribution $\hat p_{t,\text{len}}^{\theta}(l \mid \x_t)$ and we use its mean
\begingroup
    \setlength{\abovedisplayskip}{4pt}
\setlength{\belowdisplayskip}{4pt}
\begin{align*}
    \hat \Delta l_t^\theta(\x_t) = \E_{l \sim \hat p_{t,\text{len}}^{\theta}(\cdot \mid \x_t)}[l].
\end{align*}
\endgroup
This allows us to optimize the length term using the simpler objective of matching the observed length-to-go $n-m$ with the predicted mean $\hat \Delta l_t^\theta(\x_0[\bm b])$, namely
\begingroup
    \setlength{\abovedisplayskip}{4pt}
\setlength{\belowdisplayskip}{4pt}
\begin{align*}
    \mathcal L^{\textnormal{len}}_{\textnormal{mean}}(\theta)
    = \E\left[\hat \Delta l_t^\theta(\x_0[\bm b]) - (n-m)\log \hat \Delta l_t^\theta(\x_0[\bm b])\right],
\end{align*}
\endgroup
which differs from the corresponding Poisson negative log-likelihood only by an additive constant independent of $\theta$.

\xhdr{Sampling.} Given the rate $\hat R_t^\theta$, the generation is simulated using a simple first-order (Euler) discretization (\cref{eq:ctmc:first_order_discretization}) as described in \Cref{alg:prediction1}.
We call the combination of joint noising and the Euler sampling \textbf{D}iffusion-based \textbf{I}nsertion \textbf{L}anguage Model with \textbf{S}ingle Insertions \jointhighlight{(\idlm{})}.

\xhdr{Stopping Condition.}
The Insertion Language Model (ILM) of \citet{patel2025insertion}, which uses a stopping classifier to determine when to stop generation, can be viewed as a time-agnostic special case where the Bernoulli probability $\hat \lambda^\theta_t(\x) \Delta t$ in \Cref{alg:prediction1} is replaced with a stopping probability $p_{\text{stop}}^{\theta}(\cdot \mid \x)$. 
Note, however, that ILM does not use the time parameter. Moreover, it uses a hyper-parameter $\zeta$ to determine when to stop generation using the stopping condition $p_{\text{stop}}^{\theta}(\text{stop} \mid \x) < \zeta$.
This means that there is no way to trade-off the quality of the generated sequence by reducing $\Delta t$, or equivalently, increasing the number of sampling steps. 
 Our method, like most other diffusion-based generative models, allows for this trade-off.

\subsubsection{Independent Insertion}\label{sec:independent_insertion_probabilities_parameterization}
For the case of \independenthighlight{independent noising}, we parameterize the rate of the generative process also using independent, per-position rates as
\begingroup
    \setlength{\abovedisplayskip}{4pt}
\setlength{\belowdisplayskip}{5pt}
\begin{align*}%
    \hat R_t^{\theta}(\x, \y) = \sum_{i=1}^{m+1} \sum_{w \in \vocab} \lambda_t^{\theta}(\x, i)\,\hat p_{t}^{\theta}( w \mid \x, i)~ \delta_{\ins(\x, i, w)}(\y).
\end{align*}
\endgroup
In this case, with a slight abuse of notation, we use the same symbol $\hat p_t^\theta$ as before but now $\hat p_t^{\theta}: \sX\times [0,1] \times \natset{n} \to \vocab$ denotes the categorical distribution over vocabulary item $w$ for gap $i$.
Applying the same approximations as done for \Cref{cor:elbo}, we get the following estimate of the upper bound on the negative log-likelihood, which is biased for the same reasons as in \Cref{cor:elbo}.
\begin{restatable}[ELBO: Independent Insertion]{corollary}{CorELBOSecondInd}\label{cor:elbo_independent_insertion_probabilities}
    A (biased) estimate of the upper bound on the negative log-likelihood in \Cref{prop:base_elbo} for the independent per-position parameterization of $\hat R_t^\theta$ is given by
    $\mathcal L(\theta) = \mathcal L_{\textnormal{len}}(\theta) + \mathcal L_{\textnormal{token}}(\theta)$, where the two loss terms are
    \begin{align*}
        &\E\left[\sum_{i=1}^{m+1} \hat \lambda_t^\theta(\x_0[\bm b], i) - \gamma_t^{\text{ind}} |s_i(\bm b)| \log \hat \lambda_t^\theta(\x_0[\bm b], i)\right],  \\
        & -\E\left[\gamma_t^{\text{ind}} \sum_{i=1}^{m+1} \sum_{k \in s_i(\bm b)} \log \hat p_{t}^{\theta}(x_0^k \mid \x_0[\bm b], i)\right].
    \end{align*}
    Here, the expectation is over $\x_0 \sim p_{\textnormal{data}}$, $t\sim \textnormal{Uniform}[0, T]$, $m \sim \textnormal{Binomial}(n, \rho_t)$, $\bm b \sim \textnormal{Uniform}(\sB_{n, m})$, $\gamma_t^{\text{ind}} = \frac{\sigma_t\,\rho_t}{1-\rho_t}$ with $\rho_t = \exp(-\bar \sigma_{0,t})$,
    $\x_0=(n, \seqx_0)$, and $s_i(\bm b)$ is the set of indices of 0s in $\bm b$ that fall between the $(i-1)$-th and $i$-th 1s.
\end{restatable}
\let\CorELBOSecond\CorELBOSecondInd
As in the joint case, in the implementation we parameterize a distribution over the number of tokens inserted into each gap, $\hat p_{t,\textnormal{len}}^{\theta}(l \mid \x_t, i)$, with $\hat \Delta l_t^\theta(\x_t, i)\coloneqq \E_{l \sim \hat p_{t,\textnormal{len}}^{\theta}(\cdot \mid \x_t, i)}[l].$
The corresponding reverse-process insertion rate is then
    $\hat \lambda_t^\theta(\x_t, i)\coloneqq \gamma_t^{\textnormal{ind}} \hat \Delta l_t^\theta(\x_t, i).$
Consequently, the length term can be optimized by matching the observed gap size $|s_i(\bm b)|$ with the predicted mean using $\mathcal L^{\textnormal{len}}_{\textnormal{mean}}(\theta)$ given by
\begingroup
    \setlength{\abovedisplayskip}{4pt}
\setlength{\belowdisplayskip}{4pt}
    \begin{align*}
     &\E\left[\gamma_t^{\textnormal{ind}} \sum_{i=1}^{m+1} \left(\hat \Delta l_t^\theta(\x_0[\bm b], i) - |s_i(\bm b)| \log \hat \Delta l_t^\theta(\x_0[\bm b], i)\right)\right],
\end{align*}
\endgroup
which recovers the same objective up to additive constants independent of $\theta$.

\xhdr{Neural Network.} We use the same transformer architecture as in the case of Joint Insertion Parameterization, but now the token head is interpreted independently at each gap:
\begin{align*}
    &s^{\theta}(w \mid \x_t, i) = \mlpins\left(\dec(\tilde \x_t; t)^{i}\right)^{w} \\
    &\hat p_t^{\theta}(w \mid \x_t, i) = \frac{\exp(s^{\theta}(w \mid \x_t, i))}{\sum_{w'\in \vocab} \exp(s^{\theta}(w' \mid \x_t, i))}.
\end{align*}
That is, unlike the joint parameterization, the softmax is taken only over vocabulary items independently for each gap. To parameterize the distribution over the length of each gap, we use an additional head that outputs $\hat p_{t,\textnormal{len}}^{\theta}(l \mid \x_t, i)$. 

\xhdr{Sampling.}
\Cref{alg:prediction-multi} shows an efficient constrained $\tau$-leaping~\citep{gillespie2001approximate} based sampling procedure for the independent parameterization, which can insert multiple tokens per step across different gaps. Additionally, this algorithm can be applied on a mini-batch of sequences in parallel.\footnote{We provide a simplified implementation of \Cref{alg:prediction-multi} in PyTorch in Appendix~\ref{app:lst:multi_token_step}.}
We call the parameterization \independenthighlight{\idlmm{}}, where \emph{M} stands for multiple insertions across gaps.

\begin{algorithm}[t]
    \small
    \caption{Sampling from \jointhighlight{\idlm{}}}\label{alg:prediction1}
    \begin{algorithmic}[1]
        \REQUIRE Sequence $\x=(m, (\vv, \vu))$, steps $\algvarsc{MaxSteps}$
        \STATE $\Delta t \gets \frac{1}{\algvarsc{MaxSteps}}$
        \WHILE{$t < 1$}
            \STATE $e \gets \Uniform[0, 1]$
            \STATE $\hat \lambda_t^\theta(\x) \gets  \frac{\sigma_t}{\bar \sigma_{0,t}\,} \E_{l \sim p_{t,\, \text{len}}^{\theta}(\x)} \frac{l}{\tilde S_{l,m,\sigma_t}}$
            \IF{$e\leq \hat \lambda^\theta_t(\x)\, \Delta t$}
                \STATE $i, w \sim \hat p_t^{\theta}(i, w \mid \x)$
                \STATE $\vv' \gets \texttt{concat}(\vv, w)$
                \FOR{$k=1$ to $\texttt{len}(\vu)$}
                    \IF{$\vu[k] > i$}
                        \STATE $\vu[k] \gets \vu[k]+1 $
                    \ENDIF
                \ENDFOR
                \STATE $\vu' \gets \texttt{concat}(\vu, i+1)$
                \STATE $\x \gets (\vv', \vu')$
            \ENDIF
            \STATE $t \gets t + \Delta t$
        \ENDWHILE
        \RETURN~ $\x$
    \end{algorithmic}
\end{algorithm}
\begin{algorithm}
    \small
    \caption{Sampling from {\independenthighlight{\idlmm{}}}}\label{alg:prediction-multi}
    \begin{algorithmic}[1]
    \REQUIRE Sequence $\x=(x_0,\dots,x_{m-1})$, time $t\in[0,1]$, steps $\algvarsc{MaxSteps}$
    \STATE $\Delta t \gets \frac{1}{\algvarsc{MaxSteps}}$
    \WHILE{$t<1$}
        \STATE Initialize $\mathbf{r}\in\{0,1\}^{m}$ and $\mathbf{w}\in\mathcal{V}^{m}$
        \FOR{$i \in \{0,1,\dots,m-1\}$}
            \STATE $e \sim \Uniform[0,1]$
            \IF{$e \le 1 - \exp(-\hat\lambda_t^\theta(\x,i)\,\Delta t)$}
                \STATE $\mathbf{r}[i] \gets 1$
                \STATE $\mathbf{w}[i] \sim \hat p_t^\theta(w \mid \x,i)$
            \ELSE
                \STATE $\mathbf{r}[i] \gets 0$
            \ENDIF
        \ENDFOR
        \STATE $\x' \gets [\texttt{PAD},\dots,\texttt{PAD}]$
        \STATE $\mathbf{s} \gets \algvarsc{RollRight}(\mathbf{r},1)$
        \STATE $\texttt{idx\_old} \gets \algvarsc{Arange}(m) + \algvarsc{CumSum}(\mathbf{s})$
        \STATE $\texttt{idx\_ins} \gets \algvarsc{Arange}(m) + \algvarsc{CumSum}(\mathbf{r})$
        \STATE $\x' \gets \algvarsc{Scatter}(\x', \texttt{idx\_old}, \x)$
        \STATE $\x'[\texttt{idx\_ins}[\mathbf{r}]] \gets \mathbf{w}[\mathbf{r}]$
        \STATE $t \gets t + \Delta t$; $\x \gets \x'$
    \ENDWHILE
    \RETURN~ $\x$
    \end{algorithmic}
\end{algorithm}

%% file: drafts/aistats/sections/02_related_work.tex
\section{\MakeUppercase{Related Work}}
\label{sec:related work}
\xhdr{Discrete diffusion and flow matching.} 
Discrete diffusion models~\citep{sahooSimpleEffectiveMasked2024,shi2024simplified,austinStructuredDenoisingDiffusion2021,gongScalingDiffusionLanguage2024}, inspired from diffusion models~\citep{hoDenoisingDiffusionProbabilistic2020,sohl-dicksteinDeepUnsupervisedLearning2015,songGenerativeModelingEstimating2019}, generate sequences by iteratively \emph{denoising} the tokens of a noisy sequence by framing both the noising and the denoising processes as Markov chains on countable state spaces.
The discrete diffusion models can only work with fixed length sequences. They generate variable length sequences by modeling special \texttt{PAD} token as an ordinary token.
Our formulation is motivated by the discrete diffusion framework, but with some key distinctions. We do not assume independent noise per-component as done in \citet{campbell2022a} and subsequent works, instead our noising process acts directly on the sequence.

\xhdr{Insertion-based sequence generation.} 
There have been several works in the machine translation literature that explore insertion-style generation in sequence-to-sequence models~\citep{sternInsertionTransformerFlexible2019a,gu_levenshtein_2019,ruis2020insertiondeletion,patel2025insertion}.
\citet{welleck2019nonmonotonic} uses reinforcement learning to learn an insertion policy using the \emph{learning to search} approach~\citep{pmlr-v15-ross11a}.
The insertions in \citet{welleck2019nonmonotonic} are constrained to be level order traversals of a tree and cannot be arbitrary. Moreover, reinforcement learning can be orders of magnitude slower than our approach and is therefore not suitable for pre-training language models.
The insertion transformer \citep{sternInsertionTransformerFlexible2019a} uses an efficient denoising objective to train the insertion model.
It can be viewed as implementing independent insertions rate matrix (\cref{sec:independent_insertion_probabilities_parameterization}). However, it does not parameterize the insertion rate separately from the token probabilities.
In general, deciding \emph{when to stop the generation} is challenging for insertion-style models \citep{sternInsertionTransformerFlexible2019a, patel2025insertion,reid2022diffuserdiscretediffusioneditbased}, and formalizing the process a continuous time Markov chain allows us parameterize the model in a way that separates token predictions from the stopping probability.

\newpage

\xhdr{Concurrent work.} More recently, Discrete flow matching~\citep{campbellGenerativeFlowsDiscrete2024,lipman2024flowmatchingguidecode} and stochastic interpolants~\citep{albergo2023stochasticinterpolantsunifyingframework} have been proposed as simplified frameworks for generative modeling using Markov chains, wherein there is no noising process but only a generative Markov chain that is constrained to generate sequence of time marginals that converge to the true data distribution.
We provide an extended discussion of the relationship between our work and these frameworks in the \cref{app:extended_related_work}.

%% file: drafts/aistats/sections/05_experiments.tex
\section{\MakeUppercase{Empirical Evaluation}}
\label{sec:experiments}

We split our experiments into two parts. First, we demonstrate that the proposed approach retains the benefits of insertion-style generation by evaluating it on the variable-length planning task on the star graphs \citep{pmlr-v235-bachmann24a, patel2025insertion}.
We then evaluate the usefulness of the proposed approach on language modeling using two language modeling corpora with different characteristics.
We compare our results with the standard left-to-right autoregressive language models (ARMs), masked diffusion models~\citep{ouYourAbsorbingDiscrete2024,zheng2025masked} and the Insertion Language Model (ILM) \citep{patel2025insertion}.
\paragraph{Model} For \idlm{}, we use the same transformer backbone as in \citet{patel2025insertion}, with one change. We introduce the time parameter $t$ using adaptive layer normalization (AdaLN-Zero) \citep{peeblesScalableDiffusionModels2023}.
For \idlm{}, we also use the first token position to model $\hat \lambda_t^\theta(\x)$.
For \idlmm{}, we model the per-position insertion rate $\lambda_t^\theta(\x, i)$ using an additional feedforward layer on top of the transformer backbone. 
\subsection{Planning Task}
\begin{table}
    \centering
    \setlength{\tabcolsep}{15pt}
    \caption{Sequence-level exact match accuracy on the synthetic Star graph tasks.
    } \label{tab:star_results}
    \begin{tabular}{@{}lccc@{}}
    \toprule
    {Model}  & {\small Star\textsubscript{easy}} & {\small Star\textsubscript{medium}} & {\small Star\textsubscript{hard}} \\\cmidrule{1-4}
    ARM                     & 32.3                                                & 75.0                                                 & 23.0          \\
    MDM                     & 100.0                                               & 36.5                                                 & 21.0          \\
    ILM                     & 100.0                                               & {100.0}                                       & {99.1} \\ 
    \idlm{}                     & 100.0                                               & \textbf{100.0}                                       & \textbf{99.3} \\
    \idlmm{}                     & 100.0                                               & {98.60}                                       & {95.1} \\
    \bottomrule
    \end{tabular}%
    \end{table}
The simple star graph task \citep{pmlr-v235-bachmann24a} and its harder variable-length variant \citep{patel2025insertion} are minimal tasks designed to exemplify limitations of ARMs and MDMs. \footnote{We use the dataset from \citep{patel2025insertion}.}
This is a sequence-to-sequence task where the source sequence contains randomly permuted edges of a star-shaped graph, followed by the start and the target node. The model needs to generate the edges that connect the start and the target node going through the junction of the star.
This task is difficult for ARMs trained to predict from left-to-right using teacher forcing~\citep{pmlr-v235-bachmann24a} because the model is overwhelmed by the training signal for the easy-to-predict non-junction edges and it never learns to perform the implicit lookahead required to generate the edge that immediately follows the junction.
MDMs, which can model all the tokens simultaneously, can generate the correct sequence only when the arm length is fixed \citep{patel2025insertion}.
As shown in \Cref{tab:star_results}, both \idlm{} and \idlmm{} continue to enjoy the same benefits as ILM and achieve almost perfect accuracy on all three versions of the task.
\begin{figure*}[t]
    \centering
    \begin{subfigure}[t]{0.48\linewidth}
        \centering
        \includegraphics[width=\linewidth]{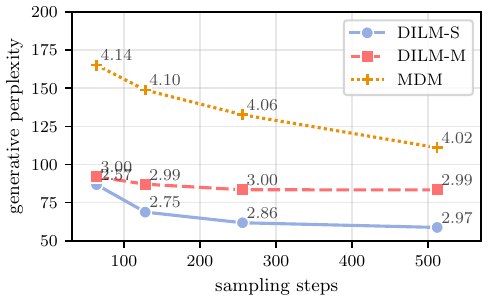}
    \end{subfigure}\hfill
    \begin{subfigure}[t]{0.48\linewidth}
        \centering
        \includegraphics[width=\linewidth]{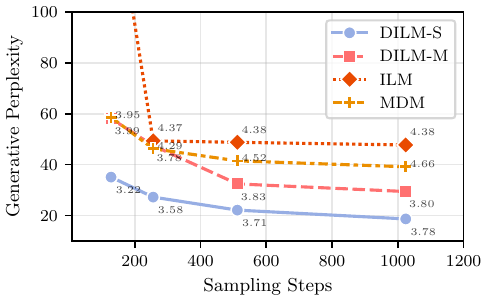}
    \end{subfigure}
    \vspace{-5pt}
    \caption{Generative perplexity vs.\ number of sampling steps for models trained on LM1B (left) and OpenWebText (right). The annotations show the token entropy values.}
    \label{fig:lm_generative_perplexity_vs_sampling_steps}
    \vspace{-8pt}
\end{figure*}
\subsection{Language Modeling}
For all the language modeling experiments, we use a transformer backbone with 12 layers, 768 hidden dimensions, and 12 attention heads with $\approx$87 million non-embedding parameters and train it from scratch using the AdamW optimizer on 8 A100 GPUs.
\begin{table}
    \centering
    \setlength{\tabcolsep}{9.9pt}
    \caption{ Evaluation of unconditional generation quality using per-token NLL under Llama 3.2 3B. The rows with the dataset names contain the NLL and entropy of the examples in the training data. 
    }
    \vspace{-5pt}
            \label{tab:lm-unconditional-results}
            \begin{tabular}{lccc}
                \toprule
                 & NLL\tiny{$\downarrow$} & Ent\tiny{$\uparrow$} & $\overline{\text{len}}$ \\
                \midrule
                \small{LM1B}          & \small{3.71}                   & \small{3.08} & \small{28} \\
                \midrule
                ARM                   & \textbf{3.94}                           & {3.12} & \small{30} \\
                MDM \tiny{\citep{sahooSimpleEffectiveMasked2024}}                    & 4.81                           & {3.70} & \small{85} \\
                ILM \tiny{\citep{patel2025insertion}}     & {4.67}                           & 2.80 & \small{21} \\
                \idlm{} \tiny{(ours)}     & \underline{4.65}                           & 2.87 & \small{30} \\
                \idlmm{} \tiny{(ours)}     & {4.76}                           & 2.99 & \small{48} \\
                \bottomrule
                \small{OpenWebText (padded)}          & 2.59                    & 5.01  & 533  \\
                \midrule
                MDM                   &  {3.96}                          & {4.64} & {320} \\
                ILM     & {4.65}                           & {5.46} & {860} \\
                \idlm{}     & \textbf{3.74}                           & {4.33} & {267} \\
                \idlmm{}     & \underline{3.92}                           & {4.54} & {478} \\
                \bottomrule
            \end{tabular}%
    \vspace{-10pt}
    \end{table}

\xhdr{LM1B} For language modeling on short sequences, we use the One Billion Word Benchmark (LM1B) \citep{chelba2013one}, which contains short sequences from the news domain.
We preprocess that data by tokenizing it with the BERT tokenizer \citep{devlin2019bertpretrainingdeepbidirectional} and pad to a maximum sequence length of 128.
We use a learning rate of $10^{-4}$, an effective batch size of 512, and train for 1 million steps.

To evaluate unconditional generation we compute per-token negative log-likelihood (NLL) under Llama-3.2-3B \citep{grattafiori2024llama} for 1000 unconditionally generated sequences.
As seen in \Cref{tab:lm-unconditional-results}, both \idlm{} and \idlmm{} are competitive with the ILM \citep{patel2025insertion}, with \idlm{} doing slightly better.
The key advantage of the diffusion-based formulation is the ability to trade-off sample quality and number of forward passes.
As shown in \Cref{fig:lm_generative_perplexity_vs_sampling_steps}, the generative perplexity of both \idlm{} and \idlmm{} improves as we increase the number of sampling steps.

\xhdr{OpenWebText.} 
For language modeling on longer sequences, we use the OpenWebText~\citep{Gokaslan2019OpenWeb} corpus.
We preprocess the OpenWebText corpus by tokenizing it using the GPT-2 tokenizer, removing sequences that are longer than 1024 tokens, and padding to a maximum sequence length of 1024.
We use a learning rate of $10^{-4}$, effective batch size of 512, and train for 300 thousand steps.
For evaluation, we compute the per-token negative log-likelihood under the GPT-2 Large model for 1000 unconditionally generated sequences from each model; \Cref{tab:lm-unconditional-results} reports these values.
We can see that \idlm{} and \idlmm{} obtain better NLL than ILM and MDM.
\Cref{fig:lm_generative_perplexity_vs_sampling_steps} (right) shows the generative perplexity \emph{vs.}\ number of sampling steps on OpenWebText.
Here we observe again that the diffusion-based formulation allows for better trade-off between sample quality and number of forward passes. The generative perplexity for ILM is very high at low sampling steps (128); it falls rapidly at medium sampling steps (256), and then remains flat even as the sampling steps are increased. \idlm{} and \idlmm{} both exhibit a smoother trade-off curve between the generative perplexity and the number of sampling steps.
However, on longer sequences, we observe token repetitions, which show up in the lower entropy values.
\section{\MakeUppercase{Conclusion}}
We presented a diffusion-based approach for variable-length sequence generation, leading to two new parameterizations of the generative process.
With clearly stated assumptions, we derived a low-variance learning objective for both parameterizations, and demonstrated their efficacy empirically on planning and language modeling tasks.

\xhdr{Limitations and future work.} 
While \idlmm{} can generate multiple tokens, it is still constrained to produce only one token per gap. 
This can hinder the ability of the neural network in leveraging correlations between tokens that occur in contiguous spans. 
A multi-token insertion model that can insert contiguous spans of tokens is a promising direction for further exploration.
Scaling insertion-based generation to larger models and longer sequences also remains an interesting direction for future work.

%% file: drafts/aistats/sections/z0_appendix.tex
\appendix
\thispagestyle{empty}

\onecolumn
\aistatstitle{\LARGE Appendix}
\etocdepthtag.toc{mtappendix}             %
\etocsettagdepth{mtchapter}{none}         %
\etocsettagdepth{mtappendix}{subsubsection} %
\raggedbottom
\tableofcontents
\clearpage
\flushbottom
\input{drafts/aistats/sections/appendices/notation}
\input{drafts/aistats/sections/appendices/ctmc}

\input{drafts/aistats/sections/appendices/proofs}

\input{drafts/aistats/sections/appendices/extended_related_work}

\input{drafts/aistats/sections/appendices/implementation}

%% file: drafts/aistats/sections/appendices/notation.tex
\section{Summary of Notation}\label{app:notation}

\begin{table}[H]
\setlength{\tabcolsep}{24.5pt}
\centering
\begin{tabular}{ll}
\toprule
\textbf{Notation} & \textbf{Description} \\
\midrule
\multicolumn{2}{c}{\textbf{General}} \\
\midrule
$(\Omega, \mathcal O,  P)$ & Probability space from which all random variables are drawn.  \\
$\omega$ & Sample point in $\Omega$ \\
$\cX_t$, $\cB_t$, $\cD_t$ & Random variables (capital letters) \\
$\x$, $\b$, $d$ & Values of random variables (lowercase) \\
$\bm{X}$, $\bm{x}$ & Non-scalar random variables and values (boldface) \\
$\sX$, $\sB$ & Sets (blackboard font) \\
$\natset{n}$ & Set of natural numbers $\{1, 2, \ldots, n\}$ \\
$\Delta^n$ & Simplex of dimension $n$, i.e., $\{p \in \mathbb{R}^n \mid p \geq 0, \sum_{i=1}^n p_i = 1\}$ \\
$x^i$, $X^i$ & $i$-th component of $\x$ and $\cX$ respectively \\
$\cX_t$, $\cB_t$, $\cD_t$ & Continuous time stochastic processes \\
$\dX_k$, $\dB_k$ & Discrete time stochastic processes \\
$p_{s|t}(\x \mid \x')$ or $p_{t,s}(\x', \x)$ & Transition probability $\sP(\cX_s = \x \mid \cX_{t} = \x')$ \\
$\ubar p_{k|k'}(\x \mid \x')$ & Transition probability $\sP(\dX_k = \x \mid \dX_{k'} = \x')$ \\
$\delta_{\sA}(a)$ & Indicator function for set $\sA$: 1 if $a\in \sA$, 0 otherwise \\
$\delta_{b}(a)$ & Shorthand for $\delta_{\{b\}}(a)$ \\
$\text{Diag}(\lambda)$ & Diagonal matrix with diagonal elements given by $\lambda$ \\
\midrule
\multicolumn{2}{c}{\textbf{Specific variables}} \\
\midrule
$\vocab$ & Vocabulary of tokens \\
$\vocab^n\coloneqq \vocab \times \cdots \times \vocab$ & Set of token sequences of length $n$. $\vocab^0 \coloneqq \emptyset$ by convention. \\
$\sX_n$ & $\{n\}\times \vocab^n$ \\
$\sB_{n, m}$ & Set of all bit vectors of length $n$ with exactly $m$ 1s. \\
$\sB_{n}$ & $\bigcup_{m=0}^{n} \sB_{n, m}$, i.e., set of all bit vectors of length $n$. \\
$|\cdot|: \sB_{n} \to \mathbb{N}$ & The number of 1s in a bit vector; $|\bm b| = \sum_{i=1}^n b^i$ \\
$\Idx(\bm b) \coloneqq \{i \in \natset{n} \mid b^i = 1\}$ & The set of indices of the ones in $\bm b$ \\
$\del(\x, i)$ & $\del(x^1\dots x^i\dots x^n, i) = x^1\dots x^{i-1} x^{i+1}\dots x^n$ \\
$\ins(\x, i, a)$ & $\ins(x^1\dots x^i\dots x^n, i, a) = x^1\dots x^{i-1} a\, x^{i+1}\dots x^n$ \\
$\rins(\x)$ & For $\x=(n, \seqx)$, $\rins(\x) = \{\ins(\x, i, a) \mid i\in \natset{n+1}, a\in \vocab\} $ \\
\bottomrule
\end{tabular}
\caption{Summary of notation used throughout the paper.}
\label{tab:notation}
\end{table}

%% file: drafts/aistats/sections/appendices/ctmc.tex
\section{Background: Continuous-Time Markov Chains on Countable State Spaces}\label{app:ctmc}
In this section, we will review some elementary results from the theory of CTMCs on countable state spaces that are used in the main paper text. Most of these results are well known and can be found in introductory textbooks on stochastic processes like \citet{cinlarIntroductionStochasticProcesses2013}.
\subsection{Structure of CTMCs}
Let $X_t$ be a CTMC taking values in a countable state space $\sX$.

\begin{definition}[Waiting time]\label{def:ctmc:waiting_time}
    The random variable $W_t(\omega)=\inf\{s>0: X_s(\omega)\neq X_t(\omega)\}$ is called the waiting time of the CTMC.
\end{definition}

The following proposition states that the conditional distribution of the waiting time is exponential and only depends on the current state.
\begin{fact}[Conditional Distribution of Waiting Time]\label{prop:ctmc:conditional_distribution_of_waiting_time}
    For any $x\in \sX$, and $t\geq 0$,
    \begin{align*}
        \sP(W_t > u \mid X_t=x) = \exp \left({-\int_{t}^{t+u} \lambda(x, s) \dd s } \right), ~~u \geq 0,
    \end{align*}
    where $\lambda(x, s) \in [0, \infty]$, with the convention that if $\lambda(x, s)=\infty$, then $\sP(W_t > u \mid X_t=x)=0$ for all $u\geq 0$.
\end{fact}
\begin{proof}
    Note that $\{W_t > u+v\} = \{W_t > u, W_{t+u} > v\}$ because $\omega \in \{ W_t > u+v \}$ implies $X_t(\omega)=X_{t+(u+v)}(\omega)=X_{t+u}(\omega)$, which further implies that $\omega \in \{ W_{t} > u \}$ and $\omega \in \{ W_{t+u} > v \}$.
    Therefore,
    \begin{align*}
        &\sP(W_t > u+v \mid X_t=x) \\
        &= \sP(W_t > u,~ W_{t+u} > v \mid X_t=x)  \\
        &= \sP(W_t > u \mid X_t=x) \sP(W_{t+u} > v \mid X_t=x, W_t > u) \\
        &= \sP(W_t > u \mid X_t=x) \sP(W_{t+u} > v \mid X_{t+u}=x) ~~~\text{(Markov property)}\\
    \end{align*}
    Denoting $\sP(W_t > u \mid X_t=x)$ as $p(x, u; t)$, we have
    \begin{align}\label{app:eq:ctmc:1}
        p(x, u+v;\, t) &= p(x, u;\, t)~ p(x, v;\, t+u).
    \end{align}
    Setting $u=0$ we get $p(x, v;\, t) = p(x, 0;\, t)~ p(x, v;\, t)$, which implies that $p(x, 0;\, t) = 1$ since $p(x, v;\, t)$ is not identically zero.
    This, along with \Eqref{app:eq:ctmc:1},  implies 
    \begin{align*}
        p(x, v;\, u) &= \frac{p(x, u+v;\, 0)}{p(x, u;\, 0)}\\
        \implies \log p(x, v;\, u) &= \log p(x, u+v;\, 0) - \log p(x, u;\, 0)\\
        \implies p(x, v;\, u) &= \exp \left({-\int_{u}^{u+v} \lambda(x, s) \dd s } \right), 
    \end{align*}
    where $\lambda(x, s) \coloneqq -\frac{\dd}{\dd s} \log p(x, s;\, 0)$, i.e., $\int_{u}^{u+v} \lambda(x, s) \dd s = \log p(x, u+v;\, 0) - \log p(x, u;\, 0)$.
\end{proof}

Next we will show that a CTMC with right-continuous sample paths can be decomposed into a product of Poisson process arrivals and DTMC transitions.
\begin{definition}[Right-continuous CTMC]\label{app:def:ctmc:right_continuous}
    A CTMC $X_t$ is said to have right-continuous sample paths if for any $t\in [0, \infty)$,
    \[
        \lim_{s\downarrow t} X_s = X_t.
    \]
\end{definition}

\begin{definition}[Types of States]\label{app:def:ctmc:types_of_states}
    Let $X_t$ be a CTMC with rate parameter $\lambda(x, s)$ as described in Proposition \ref{prop:ctmc:conditional_distribution_of_waiting_time}.
    Then a state $x\in \sX$ is called
    \begin{itemize}
        \item \emph{absorbing} if $\lambda(x, s) = 0$ for all $s\geq 0$,
        \item \emph{stable} if $0<\lambda(x, s)<\infty$ for all $s\geq 0$,
        \item \emph{instantaneous} if $\lambda(x, s)=\infty$ for some $s\geq 0$.
    \end{itemize} 
\end{definition}
\begin{fact}[Right Continuity]\label{prop:ctmc:right_continuity}
    A CTMC $X_t$ has right-continuous sample paths if there are no instantaneous states.
\end{fact}

\begin{figure}
    \centering
    \includegraphics[width=0.5\textwidth]{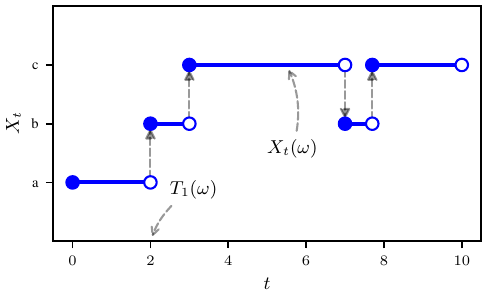}
    \caption{Sample path of a right-continuous CTMC with states $\sX=\{a, b, c\}$.}
    \label{app:fig:ctmc:right_continuous_sample_path}
\end{figure}
\Figref{app:fig:ctmc:right_continuous_sample_path} shows an example of a sample path of a right-continuous CTMC with states $\sX=\{a, b, c\}$.
Assume for the rest of this section that the CTMC is right-continuous, i.e., the mapping $t\mapsto X_t(\omega)$ is right-continuous for almost all $\omega\in \Omega$ and there are no instantaneous states. Moreover, $\sX$ is countable and has discrete topology, $\Delta$ is the point at infinity if $\sX$ is not finite. If $\sX$ is finite, then we don't need $\Delta$ and the one-point compactification of $\sX$ is the same as $\sX$.
Let $\dT_n$ be the time of the $n$-th jump of the CTMC, with $\dT_0=0$, and let $\ubar{X}_n = X_{\dT_n}$ when $\dT_n < \infty$, and $\ubar{X}_n = \ubar{X}_{n-1}$ when $\dT_n = \infty$, where the last condition is needed to handle the case of an absorbing state.
Then the following holds for all $n\geq 1$
\begin{align*}
    \dT_0 =0~;~~~~W_{\infty} = \infty~;~~~~
    \dT_n &= \dT_{n-1} + W_{\dT_{n-1}}; ~~~~
    \ubar{X}_n = \begin{cases}
        X_{\dT_n} & \text{if } \dT_n < \infty,\\
        \ubar{X}_{n-1} & \text{if } \dT_n = \infty.
    \end{cases}
\end{align*}

The next proposition clarifies the structure of CTMC $X_t$ in terms of the arrival times $\dT_n$ and the transitions $\ubar{X}_n$.
Specifically, it shows that $\ubar{X}_n$ is a DTMC and the waiting times $\dT_n - \dT_{n-1}$ depend only on the current state $\ubar{X}_{n-1}$, and are independent of the past.
\begin{fact}[Structure of CTMC]\label{prop:ctmc:structure}
    Let $X_t$ be a right-continuous CTMC with states $\sX$. 
    Then for all $n \in \sN$, and $x'\in \sX$, and $u\in \sR_+$, we have
    \begin{align*}
        &\sP(\ubar{X}_{n+1} = x',\, \dT_{n+1} - \dT_n > u \mid  \ubar{X}_0, \ldots, \ubar{X}_n, \dT_0, \ldots, \dT_n) 
        =K_{t+u}(x' \mid x)~~e^{-\Lambda(x, u;\,t)},
    \end{align*}
    if $\{\ubar{X}_n = x,~~ \dT_n = t\}$ occurs. Here $\Lambda(x, u;\,t) = \int_{t}^{t+u}\lambda(x, s) \dd s$, $K_{t+u}(x' \mid x) \geq 0$, and  $\sum_{x'\in \sX} K_{t+u}(x' \mid x) = 1$.
\end{fact}
\begin{proof}
    First we will convert the desired probability into one that uses the continuous-time variables and then apply Proposition~\ref{prop:ctmc:conditional_distribution_of_waiting_time}. 
    First note that knowing $\ubar{X}_0, \ldots, \ubar{X}_n$ and $\dT_0, \ldots, \dT_n$ is equivalent to knowing the values of $X_s$ for all $s \leq \dT_n$.
    Therefore,
    \begin{align*}
        &\sP(\ubar{X}_{n+1} = x',\, \dT_{n+1} - \dT_n > u \mid  \ubar{X}_0, \ldots, \ubar{X}_n, \dT_0, \ldots, \dT_n) \\
        &=\sP(X_{\dT_{n}+ W_{\dT_{n}}} = x',\, W_{\dT_{n}} > u \mid  X_s,\, s\leq \dT_n) \\
    \end{align*}
    Now we will use the fact that $\dT_n$ is a stopping time, because the event $\{\dT_n < t\}$ can be determined by knowing the values of $X_s$ for $s\leq t$; specifically, $\{\dT_n \leq t\}$ occurs if and only if there exists $0 < s_1, \ldots, s_n \leq t$ such that $X_{s_i} \not = X_{s_{i+1}}$ for all $i=1, \ldots, n-1$, which can be determined by knowing the values of $X_s$ for all $s\leq t$.
    Therefore, using the strong Markov property, we have
    \begin{align*}
        &\sP(X_{\dT_{n}+ W_{\dT_{n}}} = x',\, W_{\dT_{n}} > u \mid  X_s,\, s\leq \dT_n)\\
        &=\sP(X_{\dT_{n}+ W_{\dT_{n}}} = x',\, W_{\dT_{n}} > u \mid  X_{\dT_n}) \\
        &=\sP( X_{\dT_{n}+ W_{\dT_{n}}} = x' \mid X_{\dT_n},\, W_{\dT_n} > u) ~\sP(W_{\dT_n} > u \mid X_{\dT_n}) \\
    \end{align*}
    If $\dT_n = t$ and $X_{\dT_n} = \ubar{X}_n = x$, then the rightmost expression is equal to $\exp(-\Lambda(x, u;\,t))$ by Proposition~\ref{prop:ctmc:conditional_distribution_of_waiting_time}. Moreover, since $\{X_t=x,\, W_t > u\} = \{X_{t+s}=x,\, s \leq u\}$, we have
    \begin{align*}
        \sP(X_{t+ W_{t}} = x' \mid X_t = x, W_t > u) &= \sP(X_{t + u + W_{t+u}} = x' \mid X_{t+s} = x,\, s \leq u)\\
        &= \sP(X_{(t+u) + W_{t+u}} = x' \mid X_{t+u} = x)\\
        &= K_{t+u}(x' \mid x)
    \end{align*}
\end{proof}

\begin{fact}
    Denote $\sP(X_{t+u}=x' \mid X_t=x)$ as $p_{t+u|t}(x' \mid x)$. Then,
    \begin{align*}
        p_{t+u|t}(x' \mid x) = e^{-\Lambda(x, u;\,t)} \delta(x', x) + \int_{t}^{t+u} \lambda(x, s)\, e^{-\Lambda(x, s-t; t)} \sum_{y\in \sX} K_{s}(y\mid x)\, p_{t+u|s}(x' \mid y) \dd s.
    \end{align*}
\end{fact}
\begin{proof}
    Assume that $n-1$ transitions have occurred by time $t$. Then,
    \begin{align*}
        p_{t+u|t}(x' \mid x) &= \sP(X_{t+u}=x'\mid X_t=x)\\
        &= \sP(X_{t+u}=x', \dT_n - t > u \mid X_t=x) + \sP(X_{t+u}=x', \dT_n - t \leq u \mid X_t=x)\\
    \end{align*}
    The first term on the right-hand side is equal to $e^{-\Lambda(x, u;\,t)} \delta(x', x)$ because
    \begin{align*}
        \sP(X_{t+u}=x', \dT_n - t > u &\mid X_t=x) = \sP(X_{t+u}=x', \dT_n - t > u \mid X_{\dT_n}=x)\\
        &= \sP(X_{t+u}=x'\mid \dT_n -t > u,\, X_{\dT_n}=x)~\sP(\dT_n - t > u \mid X_t=x)\\
        &= \delta(x', x) ~e^{-\Lambda(x, u;\,t)} ~~~~\text{(By Proposition~\ref{prop:ctmc:conditional_distribution_of_waiting_time})}
    \end{align*}
\end{proof}

 {We can take the derivative of the expression above w.r.t.\ $u$ to get Kolmogorov equations and show the form of the rate matrix decomposed into $\lambda$ and $K$.}

\begin{fact}[Kolmogorov Forward Equation]
    Let $X_t$ be a right-continuous CTMC with states $\sX$. Then, for all $t\in [0, \infty)$, $x, x'\in \sX$, and $v>t$,
    \begin{align*}
        \frac{\partial}{\partial v} p_{v|t}(x' \mid x) = \sum_{y\in \sX} p_{v|t}(x' \mid y)~ R_v(y\mid x),
    \end{align*}
    where
    \begin{align*}
        R_v(y\mid x) = \begin{cases}
            \lambda(x, v)K_v(y\mid x) & \text{if } x\not = y,\\
            -\lambda(x, v) & \text{if } x = y.
        \end{cases}
    \end{align*}
    is the rate matrix of the CTMC, $K_v(y\mid x)$ is the transition kernel of the embedded DTMC, and $\lambda(x, v)$ is the rate parameter of the CTMC.
\end{fact}

\subsection{Time Reversal}
\begin{definition}[Time Reversal] The time reversal of a \emph{regular} CTMC $X_t$  on a finite horizon $[0, T]$ is defined as $\hat X_t$ such that $\hat X_t = X_{T-t}$ for all $t\in [0, T]$ except the jump times. For the jump times, $T_n$, $\hat X_{T_n} = X_{T-T_n^-}$.
\end{definition}

\begin{fact}[Time Reversal of CTMC] The time reversal of a regular CTMC $X_t$ is also a regular CTMC with the rate matrix 
    \begin{align*}
        \hat R_t(x, y) = \frac{p_{T-t}(y)}{p_t(x)} R_{T-t}(y, x).
    \end{align*}
\end{fact}

%% file: drafts/aistats/sections/appendices/proofs.tex
\section{Proofs}\label{app:proofs}

\subsection{ELBO}

\ThmELBO*
\input{drafts/preprint/sections/appendices/elbo.tex}

\CorELBO*

\begin{proof}\label{proof:cor_elbo}
    Replacing the rate matrices in \Cref{prop:base_elbo} with their respective components, we get
    \begin{align*}
         & \sum_{\y \not=\x} \left[-\hat R_t(\x, \y) + \frac{p_{t|0}(\y\mid \x_0)}{p_{t|0}(\x \mid \x_0)}R_t(\y, \x)\log \hat R_t(\x, \y)\right]                                                                                 \\
         & = \sum_{\y \not=\x} \left[-\hat \lambda_t(\x)\, \K[t][\hat]{\x}{\y} + \frac{p_{t|0}(\y\mid \x_0)}{p_{t|0}(\x \mid \x_0)} \lambda_t(\y)\, \K[t]{\y}{\x}\log \left(\hat \lambda_t(\x)\, \K[t][\hat]{\x}{\y}\right)\right] \\
         & = \underbrace{-\hat \lambda_t(\x)  + \log(\hat \lambda_t(\x)) \sum_{\y \not=\x} K_t(\y, \x) \frac{p_{t|0}(\y\mid \x_0)}{p_{t|0}(\x \mid \x_0)} \lambda_t(\y)}_{\text{the length term} (A)}                            \\
         & ~~~~~ + \underbrace{\sum_{\y \not=\x} K_t(\y, \x) \frac{p_{t|0}(\y\mid \x_0)}{p_{t|0}(\x \mid \x_0)} \lambda_t(\y) \log \left(\K[t][\hat]{\x}{\y}\right)}_{\text{the token term} (B)}
    \end{align*}

    \textbf{The length term (A):}

    Let $\x=(m, \seqx)$, $\y=(r, \seqy)$, and $\x_0=(n, \seqx_0)$ for some $n > r > m$. Then
    \begin{align*}
        \frac{p_{t|0}(\y\mid \x_0)}{p_{t|0}(\x \mid \x_0)} & = \frac{\sP(\dN_t=r\mid \dN_0=n) \sP(\dot \dX_{n-r}=\seqy\mid \dot\dX_0=\seqx_0)}{\sP(\dN_t=m\mid \dN_0=n) \sP(\dot \dX_{n-m}=\seqx\mid \dot\dX_0=\seqx_0)} \\
    \end{align*}

    For $t>s$, and $m>0$, from \eqref{eq:dropping_process_transition_probability} we have
    \begin{align*}
        \frac{\sP(\cN_t=m+1 \mid \cN_s=n)}{\sP(\cN_t=m \mid \cN_s=n)} & = \frac{ e^{-\bar \sigma_{s, t}} (\bar \sigma_{s, t})^{n-m-1}}{(n-m-1)!} \frac{(n-m)!}{ e^{-\bar \sigma_{s, t}} (\bar \sigma_{s, t})^{n-m}} \\
        & = \frac{n-m}{\bar \sigma_{s, t}}.
    \end{align*}
    For $m=0$, we need the tail of the Taylor series for the exponential to compute the denominator:
    \begin{align*}
        \frac{\sP(\cN_t=m+1 \mid \cN_s=n)}{\sP(\cN_t=m \mid \cN_s=n)} & =  \frac{\sP(\cN_t=1 \mid \cN_s=n)}{\sP(\cN_t=0 \mid \cN_s=n)} \\
        &= \frac{e^{-\bar \sigma_{s, t}} (\bar \sigma_{s, t})^{n-1}}{(n-1)!} \frac{1}{1 - \sum_{k=0}^{n-1} \frac{ e^{-\bar \sigma_{s, t}} (\bar \sigma_{s, t})^{k}}{k!}}\\
        &= \frac{\bar \sigma_{s,t}^{n-1}}{(n-1)!}~\frac{1}{e^{\bar \sigma_{s,t}} - \sum_{k=0}^{n-1} \frac{  (\bar \sigma_{s, t})^{k}}{k!}} \\
        &=\frac{\bar \sigma_{s,t}^{n-1}}{(n-1)!}~\frac{1}{\sum_{k=0}^{\infty} \frac{  (\bar \sigma_{s, t})^{n+k}}{(n+k)!}} \\
        &=\frac{n}{\bar \sigma_{s,t} }~\frac{1}{\sum_{k=0}^{\infty} \frac{  n!(\bar \sigma_{s, t})^{k}}{(n+k)!}} \\
    \end{align*}
    Using $S_{n,\sigma_t} = \sum_{k=0}^{\infty} \frac{n!(\bar \sigma_{s,t})^{k}}{(n+k)!}$, and $\tilde S_{n,m,\sigma_t} = (1- \delta_0(m)) + \delta_0(m) S_{n,\sigma_t}$, we can combine the two cases to get
    \begin{align}\label{eq:final_proof_1}
        \frac{\sP(\cN_t=m+1 \mid \cN_s=n)}{\sP(\cN_t=m \mid \cN_s=n)} = \frac{n-m}{\bar \sigma_{s,t} }~\frac{1}{\tilde S_{n,m,\sigma_t}}
    \end{align}

    For $r=m+1$, we have
    \begin{align*}
         & K_t(\y, \x) \frac{p_{t|0}(\y\mid \x_0)}{p_{t|0}(\x \mid \x_0)}                                                                                                                            \\
         & = \frac{n-m}{\bar \sigma_{0, t}\, \tilde S_{n,m,\sigma_t}} \frac{K(\y, \x)\sP(\dX_{n-m-1}=\seqy\mid \dX_0=\seqx_0)}{\sum_{\seqz\in \vocab^{m+1}} \sP(\dX_{n-m-1}=\seqz\mid \dX_0=\seqx_0)K(\z, \x)}
    \end{align*}
    Due to the sum over $\y$, the length term simplifies to
    \begin{align}\label{eq:length-term}
         & -\hat \lambda_t(\x) + \sigma_t\frac{(n-m)}{\bar \sigma_{0, t}\, \tilde S_{n,m,\sigma_t}} \log \hat \lambda_t(\x) \frac{ \sum_{\y \not=\x}K(\y, \x)\sP(\dX_{n-m-1}=\seqy\mid \dX_0=\seqx_0)}{\sum_{\seqz\in \vocab^{m+1}} \sP(\dX_{n-m-1}=\seqz\mid \dX_0=\seqx_0)K(\z, \x)} \nonumber \\
         & = -\hat \lambda_t(\x) + \sigma_t\frac{(n-m)}{\bar \sigma_{0, t}\, \tilde S_{n,m,\sigma_t}} \log \hat \lambda_t(\x)
    \end{align}

    \textbf{The token term (B):}

    Now we write the token term (B) along with the outer expectation, where we will denote $\sP(\dX_r=\seqy\mid \dX_0=\seqx_0)$ as $\ubar{p}_{r|0}(\seqy\mid \seqx_0)$.
    \begin{align*}
         & \E_{\x \sim p_{t|0}(\cdot|\x_0)} \sum_{\y\not=\x} \frac{\sP(\dN_t=r\mid \dN_0=n)}{\sP(\dN_t=m\mid \dN_0=n) } \frac{K(\y, \x)\ubar{p}_{n-m-1|0}(\seqy\mid \seqx_0)}{\sum_{\seqz\in \vocab^{m+1}} \ubar{p}_{n-m-1|0}(\seqz\mid \seqx_0)K(\z, \x)} \lambda_t(\y) \log \left(\K[t][\hat]{\x}{\y}\right)                                                                                                                                                                                                                             \\
         & = \frac{\sigma_t}{\bar \sigma_{0, t}}  \E_{\x \sim p_{t|0}(\cdot|\x_0)}\frac{(n-m)}{\tilde S_{n,m,\sigma_t}} \sum_{\y\in \vocab^{m+1}} \frac{K(\y, \x)\ubar{p}_{n-m-1|0}(\seqy\mid \seqx_0)}{\sum_{\seqz\in \vocab^{m+1}} \ubar{p}_{n-m-1|0}(\seqz\mid \seqx_0)K(\z, \x)} \log \left(\K[t][\hat]{\x_0[\bm b]}{\y}\right)                                                                                                                                                                                                                                      \\
         & = \frac{\sigma_t}{\bar \sigma_{0, t}}  \E_{m, \bm b }\frac{(n-m)}{\tilde S_{n,m,\sigma_t}} \sum_{\y\in \vocab^{m+1}} \frac{\ubar{p}_{n-m-1|0}(\seqy\mid \seqx_0)\,K(\y, \x_0[\bm b]) }{\sum_{\seqz\in \vocab^{m+1}} \ubar{p}_{n-m-1|0}(\seqz\mid \seqx_0)\,K(\z, \x_0[\bm b])} \log \left(\K[t][\hat]{\x_0[\bm b]}{\y}\right)                                                                                                                                                                                                                              \\
         & = \frac{\sigma_t}{\bar \sigma_{0, t}}  \E_{m, \bm b } \frac{(n-m)}{\tilde S_{n,m,\sigma_t}} \sum_{\seqy\in \vocab^{m+1}} \frac{ \sum\limits_{\bm b' \in \sB_{n, m+1}}\frac{\delta_{\seqx_0[\bm b']}(\seqy)}{\binom{n}{m+1}}\,\sum\limits_{i=1}^{m+1}\frac{\delta_{\del(\seqy, i)}(\x_0[\bm b])}{m+1} }{\sum\limits_{\seqz\in \vocab^{m+1}} \sum\limits_{\bm b' \in \sB_{n, m+1}}\frac{\delta_{\seqx_0[\bm b']}(\seqz)}{\binom{n}{m+1}}\,\sum\limits_{i=1}^{m+1}\frac{\delta_{\del(\seqz, i)}(\x_0[\bm b])}{m+1}} \log \left(\K[t][\hat]{\x_0[\bm b]}{\y}\right) \\
         & = \frac{\sigma_t}{\bar \sigma_{0, t}}  \E_{m, \bm b } \frac{(n-m)}{\tilde S_{n,m,\sigma_t}} \frac{ \sum\limits_{\bm b' \in \sB_{n, m+1}}\sum\limits_{i=1}^{m+1}{\delta_{\del(\x_0[\bm b'], i)}(\x_0[\bm b]) \log \left(\K[t][\hat]{\x_0[\bm b]}{\x_0[\bm b']}\right)}}{\sum\limits_{\bm b' \in \sB_{n, m+1}}\sum\limits_{i=1}^{m+1}\delta_{\del(\x_0[\bm b'], i)}(\x_0[\bm b])}
    \end{align*}

    Now let's look at the numerator:

    \begin{align*}
         & \sum\limits_{\bm b' \in \sB_{n, m+1}}\sum\limits_{i=1}^{m+1}{\delta_{\del(\x_0[\bm b'], i)}(\x_0[\bm b])} \log \left(\K[t][\hat]{\x_0[\bm b]}{\x_0[\bm b']}\right)                                                                                                \\
         & = \sum\limits_{\bm b' \in \sB_{n, m+1}}\sum\limits_{i=1}^{m+1}\left[\prod_{j=0\not=i}^{m+1}\delta_{\x_0[b']^j}(\x_0[\bm b]^j)\right]\left[\sum\limits_{w\in \vocab}\delta_{w}(\x_0[\bm b']^i)\right]\, \log \left(\K[t][\hat]{\x_0[\bm b]}{\x_0[\bm b']}\right) \\
         & = \sum\limits_{i=1}^{m+1} \sum\limits_{w\in \vocab} \sum\limits_{\bm b' \in \sB_{n, m+1}}\left[\prod_{j=0\not=i}^{m+1}\delta_{\x_0[b']^j}(\x_0[\bm b]^j)\right]~\delta_{w}(\x_0[\bm b']^i)\, \log \left(\K[t][\hat]{\x_0[\bm b]}{\x_0[\bm b']}\right)           \\
         & =\sum\limits_{i=1}^{m+1} \sum\limits_{w\in \vocab} \sum\limits_{\bm a \in \sB_{n, m}(\x_0, \bm b)}~\sum_{k\in s_i(\bm a)}\delta_{w}(\x_0^k)\, \log \left(\K[t][\hat]{\x_0[\bm b]}{\x_0[\bm a \vee \bm e_k]}\right)                                              \\
         & = \sum\limits_{\bm a \in \sB_{n, m}(\x_0, \bm b)}  \sum\limits_{i=1}^{m+1} ~\sum_{k\in s_i(\bm a)} \sum\limits_{w\in \vocab}\delta_{w}(\x_0^k)\, \log \left(\hat K_t(\x_0[\bm b], \x_0[\bm a \vee \bm e_k])\right)                                                             %
    \end{align*}
    where $\bm e_k\in \sB_{n, 1}$ is the vector of length $n$ with all 0s except for the $k$-th position, $\sB_{n, m}(\x_0, \bm b) = \{ \bm a \in \sB_{n, m}~:\, \x_0[\bm a] = \x_0[\bm b]\}$, and $s_i(\bm a)$ is the set of indices of zeros in $\bm a$ that fall between the $(i-1)$-th and $i$-th 1s.
    The re-indexing in the third line above is justified by a bijection between the pairs $(\bm b', i)$ satisfying $\delta_{\del(\x_0[\bm b'], i)}(\x_0[\bm b])=1$ and the triples $(\bm a, i, k)$ with $\bm a \in \sB_{n, m}(\x_0, \bm b)$ and $k\in s_i(\bm a)$.
    To see this, fix such a pair $(\bm b', i)$.
    Since deleting the $i$-th token from $\x_0[\bm b']$ yields $\x_0[\bm b]$, there is a unique original index $k\in \natset{n}$ that is present in $\bm b'$ but not in the subsequence corresponding to $\x_0[\bm b]$.
    Let $\bm a \coloneqq \bm b' \wedge (\bm 1_n - \bm e_k)$, so that $\bm b' = \bm a \vee \bm e_k$ and $\x_0[\bm a] = \x_0[\bm b]$.
    Moreover, the index $k$ must lie in the gap of $\bm a$ corresponding to insertion location $i$, and therefore $k\in s_i(\bm a)$.
    Conversely, given any triple $(\bm a, i, k)$ with $\bm a \in \sB_{n, m}(\x_0, \bm b)$ and $k\in s_i(\bm a)$, defining $\bm b' \coloneqq \bm a \vee \bm e_k$ gives a unique predecessor mask in $\sB_{n,m+1}$ such that deleting the $i$-th token from $\x_0[\bm b']$ yields $\x_0[\bm b]$.
    Therefore, summing over valid pairs $(\bm b', i)$ is equivalent to summing over triples $(\bm a, i, k)$.
    For example, the following schematic illustrates a repeated-token case with $\x_0=\textnormal{banana}$ and $\x_0[\bm b]=\textnormal{ana}$, where multiple masks $\bm a$ give the same subsequence:
    \begin{center}
    \begin{tikzpicture}[
        x=0.82cm,
        y=0.82cm,
        every node/.style={font=\footnotesize},
        boxgray/.style={draw=cics_gray, fill=white, rounded corners=2pt, line width=0.9pt},
        boxred/.style={draw=cics_red, fill=cics_bglight_red, rounded corners=2pt, line width=1.0pt},
        lineblue/.style={draw=cics_blue, line width=1.0pt},
        textblue/.style={text=cics_blue},
        textred/.style={text=cics_red}
    ]
        \node[anchor=east] at (0.2,4.0) {$\x_0$};
        \foreach \x/\lbl in {1/b,2/a,3/n,4/a,5/n,6/a} {
            \filldraw[boxgray] (\x,3.75) rectangle +(0.62,0.48);
            \node at (\x+0.31,3.99) {\lbl};
        }
        \filldraw[boxred] (4,3.75) rectangle +(0.62,0.48);
        \node at (4.31,3.99) {a};
        \node[textred] at (4.31,4.45) {$k=4$};

        \node[anchor=east] at (0.2,3.05) {$\bm b$};
        \foreach \x/\lbl in {1/0,2/1,3/1,4/1,5/0,6/0} {
            \filldraw[boxgray] (\x,2.8) rectangle +(0.62,0.48);
            \node at (\x+0.31,3.04) {\lbl};
        }
        \node[anchor=west] at (7.0,3.04) {$\x_0[\bm b]=\textnormal{ana}$};

        \node[anchor=east] at (0.2,2.15) {$\sB_{6,3}(\x_0,\bm b)$};
        \node[anchor=west] at (1.0,2.15) {$\{011100,\;010011,\;000111\}$};

        \node[anchor=east] at (0.2,1.2) {choose $\bm a$};
        \foreach \x/\lbl in {1/0,2/1,3/0,4/0,5/1,6/1} {
            \filldraw[boxgray] (\x,0.95) rectangle +(0.62,0.48);
            \node at (\x+0.31,1.19) {\lbl};
        }
        \node[anchor=west] at (7.0,1.19) {$\bm a=010011$};
        \draw[lineblue] (3.90,1.65) -- (3.90,0.68);
        \draw[lineblue] (4.72,1.65) -- (4.72,0.68);
        \node[textblue] at (4.30,1.80) {$i=3$};

        \node[anchor=east] at (0.2,0.2) {$\bm b'$};
        \foreach \x/\lbl in {1/0,2/1,3/0,4/1,5/1,6/1} {
            \filldraw[boxgray] (\x,-0.05) rectangle +(0.62,0.48);
            \node at (\x+0.31,0.19) {\lbl};
        }
        \node[anchor=west] at (7.0,0.19) {$\bm b'=\bm a\vee \bm e_4 = 010111$};
        \node[anchor=west] at (7.0,-0.35) {$\del(\x_0[\bm b'],3)=\x_0[\bm b]$};
    \end{tikzpicture}
    \end{center}

    Note that the denominator is the same as the numerator except for the log term, and therefore the overall token term is an expectation over the uniform probability mass function $\pi(\bm a, i, k \mid \x_0, \bm b)$ on the set
    \begin{align*}
        \Omega(\x_0, \bm b) \coloneqq \{(\bm a, i, k) \mid \bm a \in \sB_{n, m}(\x_0, \bm b),~ i \in \natset{m+1},~ k \in s_i(\bm a)\}.
    \end{align*}
    At this point, before introducing any approximation to the reverse kernel, the token term is
    \begin{align*}
        \frac{\sigma_t}{\bar \sigma_{0, t}} \E_{m, \bm b} \frac{(n-m)}{\tilde S_{n,m,\sigma_t}} \E_{\bm a, i, k \sim \pi(\cdot \mid \x_0, \bm b)} \log \left(\hat K_t(\x_0[\bm b], \x_0[\bm a \vee \bm e_k])\right).
    \end{align*}
    Therefore, we get an unbiased estimate of the token term by taking a single sample $(\bm a, i, k) \sim \pi(\cdot \mid \x_0, \bm b)$. 
    In fact, up to this point we can also fix $\bm a=\bm b$ without introducing bias, provided we keep the outer expectation over $\bm b \sim \textnormal{Uniform}(\sB_{n,m})$.
    To see this, for any subsequence $\seqy \in \vocab^m$, define the equivalence class
    \begin{align*}
        C(\seqy) \coloneqq \{\bm a \in \sB_{n,m} \mid \x_0[\bm a] = \seqy\},
    \end{align*}
    and define
    \begin{align*}
        h(\bm a, \seqy) \coloneqq \frac{1}{n-m}\sum_{i=1}^{m+1}\sum_{k\in s_i(\bm a)} \log \left(\hat K_t(\seqy, \x_0[\bm a \vee \bm e_k])\right).
    \end{align*}
    The exact estimator obtained from the expectation over $\pi(\cdot \mid \x_0, \bm b)$ can then be written as
    \begin{align*}
        \frac{1}{|\sB_{n,m}|}\sum_{\seqy}\sum_{\bm b \in C(\seqy)} \frac{1}{|C(\seqy)|}\sum_{\bm a \in C(\seqy)} h(\bm a, \seqy)
        = \frac{1}{|\sB_{n,m}|}\sum_{\seqy}\sum_{\bm a \in C(\seqy)} h(\bm a, \seqy)
        = \frac{1}{|\sB_{n,m}|}\sum_{\bm b \in \sB_{n,m}} h(\bm b, \x_0[\bm b]).
    \end{align*}
    Therefore, averaging over all $(i,k)$ while fixing $\bm a=\bm b$ yields another unbiased estimator.
    The samples are correlated, but the correlation does not introduce bias.
    Thus, we can equivalently use $(n-m)$ correlated samples by fixing $\bm a$ to be the same for all samples and equal to $\bm b$, which gives the exact estimator
    \begin{align}\label{eq:token-term}
         & \frac{\sigma_t}{\bar \sigma_{0, t}} \E_{m, \bm b} \frac{(n-m)}{\tilde S_{n,m,\sigma_t}} \frac{1}{n-m}\sum\limits_{i=1}^{m+1} ~\sum_{k\in s_i(\bm b)} \sum\limits_{w\in \vocab}\delta_{w}(\x_0^k)\, \log \left(\hat K_t(\x_0[\bm b], \x_0[\bm b \vee \bm e_k])\right) \nonumber\\
         & \overset{\textcolor{blue}{(*)}}{\approx} \frac{\sigma_t}{\bar \sigma_{0, t}} \E_{m, \bm b} \frac{(n-m)}{\tilde S_{n,m,\sigma_t}} \frac{1}{n-m}\sum\limits_{i=1}^{m+1} ~\sum_{k\in s_i(\bm b)} \sum\limits_{w\in \vocab}\delta_{w}(\x_0^k)\, \log \left(\hat p_{\ins}^{\theta}(i, x_0^k \mid \x_0[\bm b])\right) \nonumber\\
         & =  \frac{\sigma_t}{\bar \sigma_{0, t}} \E_{m, \bm b} \frac{(n-m)}{\tilde S_{n,m,\sigma_t}} \sum\limits_{\substack{i\in\natset{m+1},                                                                                         \\ w \in \vocab}} p_{\ins}^{\text{target}}(i, w \mid \x_0, \bm b)\, \log \left(\hat p_{\ins}^\theta(i, w \mid \x_0[\bm b])\right).
    \end{align}
    \textcolor{blue}{(*)} Here, in the second step, we invoke the following approximation. We fix the alignment between the positions of the current sequence $\x_0[\bm b]$ and a predecessor state, and identify the predecessor state $\x_0[\bm b \vee \bm e_k]$ by the insertion location $i$ and token $x_0^k$.
    Under this approximation, the reverse kernel $\hat K_t(\x_0[\bm b], \x_0[\bm b \vee \bm e_k])$ is replaced by the parameterized insertion probability $\hat p_{\ins}^{\theta}(i, x_0^k \mid \x_0[\bm b])$.
    This is exact whenever the predecessor state determines a unique insertion location-token pair, and becomes approximate in the repeated-token case discussed below.
    If $\x=\textnormal{aa}$ and $\y=\textnormal{aaa}$, then the same predecessor state $\y$ can be obtained from $\x$ by inserting the token $\textnormal{a}$ at any of the three insertion locations. Therefore, $\hat K_t(\x,\y) = \sum_{i=1}^{3} \hat p_{\ins}^{\theta}(i, \textnormal{a} \mid \x)$, and replacing $\log \left(\hat K_t(\x,\y)\right)$ with a single aligned term $\log \left(\hat p_{\ins}^{\theta}(i, \textnormal{a} \mid \x)\right)$ is an approximation in this case.
    
    Finally, $p_{\ins}^{\text{target}}(i, w \mid \x_0, \bm b) = \frac{1}{n-m}\sum\limits_{k\in s_i(\bm b)}\delta_{w}(\x_0^k)$ can be seen as the target joint probability distribution over positions and tokens. To see this, note that $\sum_{i=1}^{m+1} s_i(\bm b) = n-m$.
    Putting \cref{eq:length-term} and \cref{eq:token-term} together, we get the final expression.
\end{proof}
\begin{remark}[Rao-Blackwellization]\label{remark:rao_blackwellization}
    Instead of using the unbiased estimator obtained by fixing $\bm a=\bm b$, we could marginalize over all $\bm a \in \sB_{n,m}(\x_0, \bm b)$ explicitly. This can be viewed as Rao-Blackwellizing the latent alignment variable $\bm a$, yielding a lower-variance estimator of the ELBO term. However, this requires solving one dynamic programming instance per $(\x_0, \bm b)$ pair to compute all possible alignments of the subsequence $\x_0[\bm b]$ and $\x_0$, which can introduce significant computational overhead during training. We leave an efficient implementation of this approach to future work.
\end{remark}

\begin{proposition}[Insertion rate from the length posterior]\label{prop:length_posterior_rate}
    Fix $t$ and a current state $\x=(m,\seqx)$. Let $L \coloneqq N_0 - N_t$ and define
    \begin{align*}
        p_{t,\textnormal{len}}(l \mid \x) \coloneqq \sP(L=l \mid \cX_t=\x).
    \end{align*}
    Then the pointwise minimizer of the conditional length term in \Cref{cor:elbo} is
    \begin{align*}
        \hat \lambda_t^\star(\x)
        = \frac{\sigma_t}{\bar \sigma_{0,t}} \E\left[\frac{L}{\tilde S_{m+L,m,\sigma_t}} \,\middle|\, \cX_t=\x\right]
        = \frac{\sigma_t}{\bar \sigma_{0,t}} \sum_{l \ge 0} p_{t,\textnormal{len}}(l \mid \x)\,\frac{l}{\tilde S_{m+l,m,\sigma_t}}.
    \end{align*}
\end{proposition}
\begin{proof}
    By \Cref{eq:length-term}, for a sample with current state $\cX_t=\x=(m,\seqx)$ and initial length $N_0=n$, the length contribution to the ELBO is
    \begin{align*}
        \hat \lambda_t(\x) - \sigma_t\frac{(n-m)}{\bar \sigma_{0,t}\,\tilde S_{n,m,\sigma_t}} \log \hat \lambda_t(\x).
    \end{align*}
    On the event $\{\cX_t=\x\}$ we have $N_t=m$, hence $n-m=L$ and $n=m+L$. Therefore,
    \begin{align*}
        \hat \lambda_t(\x) - \sigma_t\frac{(n-m)}{\bar \sigma_{0,t}\,\tilde S_{n,m,\sigma_t}} \log \hat \lambda_t(\x)
        =
        \hat \lambda_t(\x) - \frac{\sigma_t}{\bar \sigma_{0,t}} \frac{L}{\tilde S_{m+L,m,\sigma_t}} \log \hat \lambda_t(\x).
    \end{align*}
    Taking the conditional expectation with respect to $L$ given $\cX_t=\x$ yields
    \begin{align*}
        \E\left[\hat \lambda_t(\x) - \frac{\sigma_t}{\bar \sigma_{0,t}} \frac{L}{\tilde S_{m+L,m,\sigma_t}} \log \hat \lambda_t(\x) \,\middle|\, \cX_t=\x\right]
        = \hat \lambda_t(\x) - \frac{\sigma_t}{\bar \sigma_{0,t}} \E\left[\frac{L}{\tilde S_{m+L,m,\sigma_t}} \,\middle|\, \cX_t=\x\right] \log \hat \lambda_t(\x).
    \end{align*}
    Note that if $m>0$, then $\tilde S_{m+L,m,\sigma_t}=1$, so this reduces to
    \begin{align*}
        \hat \lambda_t(\x) - \frac{\sigma_t}{\bar \sigma_{0,t}} \E\left[L \mid \cX_t=\x\right] \log \hat \lambda_t(\x).
    \end{align*}
    More generally, let
    \begin{align*}
        c_t(\x) \coloneqq \frac{\sigma_t}{\bar \sigma_{0,t}} \E\left[\frac{L}{\tilde S_{m+L,m,\sigma_t}} \,\middle|\, \cX_t=\x\right].
    \end{align*}
    The conditional objective is $f(a)=a-c_t(\x)\log a$ for $a>0$. Its derivative is $f'(a)=1-\frac{c_t(\x)}{a}$ and $f''(a)=\frac{c_t(\x)}{a^2}\ge 0$, so the unique minimizer for $c_t(\x)>0$ is $a=c_t(\x)$, and if $c_t(\x)=0$ the minimum is attained at $a=0$ by continuity. Therefore,
    \begin{align*}
        \hat \lambda_t^\star(\x)
        = c_t(\x)
        = \frac{\sigma_t}{\bar \sigma_{0,t}} \E\left[\frac{L}{\tilde S_{m+L,m,\sigma_t}} \,\middle|\, \cX_t=\x\right].
    \end{align*}
    Expanding the conditional expectation over the posterior of $L$ gives
    \begin{align*}
        \hat \lambda_t^\star(\x)
        = \frac{\sigma_t}{\bar \sigma_{0,t}} \sum_{l \ge 0} p_{t,\textnormal{len}}(l \mid \x)\,\frac{l}{\tilde S_{m+l,m,\sigma_t}}.
    \end{align*}
\end{proof}

\subsection{Joint Noising}

\LemmaDeletionDTMC*

\begin{remark}\label{remark:deletion_dtmc}
    Let $\bm 1_n$ be the vector of length $n$ with all 1s.
    Assume that $\z=(l, \seqz)$ and $\x=(n, \seqx)$ such that $\x=\z[\bm b]$ for some $\bm b \in \sB_{l, n}$ with $l \geq n$.
    Then there exists a bijection $\phi_{\bm b}: \sB_{l, n - 1} \to \sB_{n, n - 1}$ such that $\kappa_{\textnormal{Uni}}(\bm a\mid \bm b) = \kappa_{\textnormal{Uni}}(\phi_{\bm b}(\bm a)\mid \bm 1_n)$, and $\z[\bm a]=\x[\phi_{\bm b}(\bm a)]$ for all $\bm a \in \sB_{l, n - 1}$ with $\kappa_{\textnormal{Uni}}(\bm a\mid \bm b) > 0$.
    This is all to say that at any point during the deletion process, we can freely replace the current sequence $(n, \seqx)$ with some longer sequence $(l, \seqz)$ and vice versa, such that $\x = \z[\bm b]$, and it will not change anything about the subsequent deletion process.

    Now note that
    \begin{align*}
         & \sum_{\x\in \sX_n} \sum_{\bm b \in \sB_{l, n}} \sum_{\bm c \in \sB_{n, n-1}} \delta_{\seqz[\bm b]}(\seqx) ~\delta_{\seqx[\bm c]}(\seqy) \\
         & \overset{(a)}{=} \sum_{\bm b \in \sB_{l, n}} \sum_{\bm c \in \sB_{n, n-1}} \delta_{\seqz[\bm b][\bm c]}(\seqy)                        \\
         & \overset{(b)}{=} \sum_{\bm b \in \sB_{l, n}}\sum_{\bm a \in \sA_{l, n-1, \bm b}}  \delta_{\seqz[\bm a]}(\seqy)                          \\
         & = \sum_{\bm a \in \sB_{l, n-1}} (l - n + 1) \delta_{\seqz[\bm a]}(\seqy) \numberthis\label{eq:dtmc_remark_last_step}
    \end{align*}
    Here, $(a)$ follows from the fact that for a specific $\bm b$, there exists a unique $\x$ such that $\x=\z[\bm b]$. Therefore, we replace $\x$ with $\z[\bm b]$.
    In (b), $\sA_{l, n-1, \bm b} = \{ \bm a \in \sB_{l, n-1} \mid \kappa_{\textnormal{Uni}}(\bm a\mid \bm b) > 0 \}$. In the final step, we get rid of the sum over $\bm b$ by first noting that $\bigcup_{\bm b \in \sB_{l, n}} \sA_{l, n-1, \bm b} = \sB_{l, n-1}$, i.e., while going over the double sum, every $\bm a \in \sB_{l, n-1}$ appears at least once. Using this we flip the order of the summation and note that for a specific $\bm a \in \sB_{l, n-1}$, there will be $l-(n-1)$ choices for $\bm b$ such that $\kappa_{\textnormal{Uni}}(\bm a\mid \bm b) > 0$ because there are $l-(n-1)$ positions with zero in $\bm a$ and exactly one of them has to become 1 to obtain a valid $\bm b$.
\end{remark}

\begin{proof}
    The transition kernel is given by
    \begin{align*}
        \K{(n, \seqx)}{(m, \seqy)} & \coloneqq
        \sP(\dX_{k+1} = (m, \seqy) \mid \dX_k = (n, \seqx) )                                                                                                                                                                                   \\
                                   & = \sum_{\bm b \in \sB_{n}} \kappa_{\textnormal{Uni}}(\bm b\mid \bm 1_n)  ~\delta_{\seqx[\bm b]}(\seqy)                                                                                                    \\
                                   & = \sum_{\bm b \in \sB_{n}} \frac{1}{|\bm 1_n|} \delta_{\{\bm 1_n \succ \cdot\}}(\bm b)\,\delta_{|\bm 1_n| -1}(|\bm b|)~\delta_{\seqx[\bm b]}(\seqy) ~~\text{(by definition of $\kappa_{\textnormal{Uni}}$)} \\
                                   & = \sum_{\bm b \in \sB_{n, m}} \frac{1}{n} \,\delta_{n-1}(m)\,\delta_{\seqx[\bm b]}(\seqy)                                                                                                                   \\
                                   & = \sum_{i=1}^{n} \frac{1}{n} \,\delta_{\del(\x, i)}(\y)
    \end{align*}

    The process is clearly Markovian because the transition probabilities only depend on the current state and any history does not change the one-step transition probability, and therefore any future transition probabilities.

    To get the $r$-step transition probability we can use the Chapman-Kolmogorov equation, which in the case of a DTMC simply means that we take the product of the one-step transition probabilities and marginalize over the intermediate states.

    Denoting $\z_0 = \x = (n, \seqx)$, $\z_r = \y = (m, \seqy)$, and $\bm b_0 = \bm 1_n$, and using the expression for $K$ from above, we get the following expression for one step of marginalization.
    \begin{align*}
         & \sum_{\z_1} \K{\z_0}{\z_1} \K{\z_1}{\z_2}                                                                                                                                                                                                                  \\
         & = \sum_{\z_1} \left[\sum_{\bm b_1 \in \sB_{n_0, n_1}} \frac{1}{n_0} \,\delta_{n_0-1}(n_1)\,\delta_{\seqx[\bm b_1]}(\seqz_1)\right] \left[\sum_{\bm b_2 \in \sB_{n_1, n_2}} \frac{1}{n_1} \,\delta_{n_1-1}(n_2)\,\delta_{\seqz_1[\bm b_2]}(\seqz_2)\right].
    \end{align*}
    As in \Cref{remark:deletion_dtmc}, we note that for a specific $\bm b_1$, there exists a unique $\z_1$ such that $\z_1=\x[\bm b_1]$. Therefore, we can replace $\z_1$ with $\x[\bm b_1]$ in the above expression and get rid of the sum over $\z_1$.
    We can also apply the conditions on $n_1$ and $n_2$ to write them in terms of $n_0$.

    \begin{align*}
         & = \sum_{\bm b_1 \in \sB_{n_0, n_0-1}} \frac{1}{n_0} \sum_{\bm b_2 \in \sB_{n_0-1,\,n_0-2}} \frac{1}{n_0-1} \delta_{(\seqz_0[\bm b_1])[\bm b_2]}(\seqz_2) \\
         & = \sum_{\bm b \in \sB_{n, n-2}} \frac{1\cdot 2}{n\,(n-1)} \delta_{\seqx[\bm b]}(\seqz_2)\numberthis\label{eq:single_step_prod}
    \end{align*}
    In the final step, we removed the summation over $\bm b_1$ using the same argument as we used in \Cref{eq:dtmc_remark_last_step}. We also renamed $\bm b_2$ to $\bm b$ to simplify the notation, and used the fact that $n_0=n$ and $\x=\z_0$.
    We can now repeat this procedure $r-1$ times to get the following expression for the $r$-step transition probability, which concludes the proof.
    \begin{align*}
         & \sP(\dX_{k+r} = (m, \seqy) \mid \dX_k = (n, \seqx))                                                                          \\
         & = \sum_{\z_1, \dots, \z_{r-1}} \prod_{i=1}^{r} \K{\z_{i-1}}{\z_{i}}                                                            \\
         & = \sum_{\z_{r-1}} \K{\z_{r-1}}{\z_{r}}\sum_{\z_{r-2}}\cdots \sum_{\z_2}\K{\z_2}{\z_3} \sum_{\z_1} \K{\z_0}{\z_1}\K{\z_1}{\z_2} \\
         & = \sum_{\bm b \in \sB_{n, m}} \frac{r!}{n\,(n-1)\cdots(n-(r-1))} \delta_{\seqx[\bm b]}(\seqy) \delta_{r}(n-m)
    \end{align*}
\end{proof}

\PropTransitionProbability*
\begin{proof}
    An equivalent way to represent the continuous-time and discrete-time noising processes is $\cX_t = (\cN_t, \cseqX_t)$, and $\dX_k = (\dN_k, \dseqX_k)$, respectively, where $\cN_t$ and $\dN_k$ are the lengths of the sequences at time $t$ and $k$, respectively. Using this, we can write the transition probability as
    \begin{align*}
         & \sP(\cX_t= (m, \seqy) \mid \cX_s= (n, \seqx))                                                                                                           \\
         & = \sP(\cN_t=m, \cseqX_t=\seqy \mid \cN_s=n, \cseqX_s=\seqx)                                                                                             \\
         & = \sP(\cseqX_t=\seqy \mid \cseqX_s=\seqx, \cN_s=n, \cN_t=m) \sP(\cN_t=m \mid \cN_s=n, \cseqX_s=\seqx)                                                 \\
         & \overset{(a)}{=} \sP(\dseqX_{n-m} = \seqy \mid \dseqX_0=\seqx, \dN_0=n) \sP(\cN_t=m \mid \cN_s=n)                                                     \\
         & \overset{(b)}{=} \sP(\dX_{n-m} = (m, \seqy) \mid \dX_0=(n, \seqx)) \sP(\cN_t=m \mid \cN_s=n)                                                     \numberthis\label{eq:ctmc_transition_probability}
    \end{align*}
    Here, in (a), $\sP(\cseqX_t=\seqy \mid \cseqX_s=\seqx, \cN_s=n, \cN_t=m) = \sP(\dseqX_{n-m} = \seqy \mid \dseqX_0=\seqx, \dN_0=n)$ because the embedded DTMC is time-homogeneous, and given the initial and final lengths, the CTMC only depends on the transitions of the embedded DTMC.
    Now note that deletions arrive with rate $\lambda_t((r, \seqx))=\sigma_t\delta\{r>0\}$, and therefore the number of deletions until there is nothing left to delete has a Poisson distribution with rate $\bar \sigma_{s, t} = \int_s^t \sigma(u) \dd{u}$.
    The general expression for $\sP(\cN_t=m \mid \cN_s=n)$ can be obtained by creating a special case for the absorbing state $m=0$ and noting that the total probability is 1:
    \begin{align}\label{eq:dropping_process_transition_probability}
        \sP( \cN_t =m  \mid \cN_s = n) = \begin{cases}
            \frac{ e^{-\bar \sigma_{s, t}} (\bar \sigma_{s, t})^{n-m}}{(n-m)!}                  & \text{if } 0 < m \leq n \\
            1 - \sum_{k=1}^n \frac{ e^{-\bar \sigma_{s, t}} (\bar \sigma_{s, t})^{n-k}}{(n-k)!} & \text{if } m = 0.
        \end{cases}
    \end{align}
    We obtain the transition probability for the CTMC by substituting the expression for $\sP(\dX_{n-m} = (m, \seqy) \mid \dX_0=(n, \seqx))$ from \Cref{lemma:deletion_dtmc}, and for $\sP(\cN_t=m \mid \cN_s=n)$ from \Cref{eq:dropping_process_transition_probability}, in \Cref{eq:ctmc_transition_probability}.
\end{proof}

\subsection{Independent Noising}\label{app:independent_deletion}

We can also obtain uniformly random deletion of positions by giving each position an independent deletion clock with rate $\sigma_t$. 
So the probability that a position is \emph{retained} at time $t$ given it was present at time $s$ is $\rho_{s,t} \coloneqq \exp(-\bar\sigma_{s,t})$, with $\bar\sigma_{s,t} = \int_s^t \sigma(u)\,\dd{u}$. At time $t$ the number of retained positions is therefore $\Binomial(n, \rho_{s,t})$, and conditional on length $m$, the mask is uniform over $\sB_{n,m}$.

\begin{proposition}[Transition probability for Independent Noising]\label{prop:transition_independent}
Let $\cX_t$ be the continuous-time process where each of the $n$ positions is deleted independently with rate $\sigma_t$. Then for $s \le t$,
\begin{align*}
p_{t|s}\bigl((m, \seqy) \mid (n, \seqx)\bigr) = \binom{n}{m} \rho_{s,t}^m (1-\rho_{s,t})^{n-m} \,\frac{1}{\binom{n}{m}} \sum_{\bm b \in \sB_{n,m}} \delta_{\seqx[\bm b]}(\seqy) = \rho_{s,t}^m (1-\rho_{s,t})^{n-m} \sum_{\bm b \in \sB_{n,m}} \delta_{\seqx[\bm b]}(\seqy),
\end{align*}
for $0 \le m \le n$, with $\rho_{s,t} = \exp(-\bar\sigma_{s,t})$. The empty sequence $(0, \emptyset)$ has mass $(1-\rho_{s,t})^n$.
\end{proposition}
\begin{proof}
For each position $i \in \natset{n}$, let $b^i \in \{0,1\}$ indicate whether that position is retained at time $t$ given it was present at time $s$. Since the deletion clock at each position has rate $\sigma_t$, the survival probability over $[s,t]$ is
\begin{align*}
    \sP(b^i = 1) = \rho_{s,t} = \exp(-\bar\sigma_{s,t}),
\end{align*}
and the indicators are independent across positions. Hence the retention mask $\bm b = (b^1, \dots, b^n)$ satisfies
\begin{align*}
    \sP(\bm b) = \rho_{s,t}^{|\bm b|}(1-\rho_{s,t})^{n-|\bm b|}.
\end{align*}
Therefore, for any $\seqy$ of length $m$,
\begin{align*}
     & p_{t|s}\bigl((m, \seqy) \mid (n, \seqx)\bigr) \\
     & = \sum_{\bm b \in \sB_n} \sP(\bm b)\,\delta_{|\bm b|}(m)\,\delta_{\seqx[\bm b]}(\seqy) \\
     & = \sum_{\bm b \in \sB_{n,m}} \rho_{s,t}^m(1-\rho_{s,t})^{n-m}\,\delta_{\seqx[\bm b]}(\seqy) \\
     & = \rho_{s,t}^m(1-\rho_{s,t})^{n-m}\sum_{\bm b \in \sB_{n,m}} \delta_{\seqx[\bm b]}(\seqy).
\end{align*}
Equivalently, since every mask in $\sB_{n,m}$ has the same probability, the number of retained positions is $\mathrm{Binomial}(n, \rho_{s,t})$, and conditional on $|\bm b|=m$, the mask is uniform over $\sB_{n,m}$, which gives the first factorization in the proposition.

\end{proof}

\input{drafts/aistats/sections/appendices/direct_parameterization.tex}

%% file: drafts/preprint/sections/appendices/elbo.tex
\begin{proof}\label{proof:elbo}
The result can be proved for any regular CTMC using Dynkin's formula \citep{hanson2007applied} for the change of measure on the CTMC path space followed by the data processing inequality as discussed in \citet{loudiscrete}.

In order to provide useful intuition for the result, we instead use elementary techniques to provide an informal argument in which we derive the continuous-time ELBO from the discrete-time ELBO, and then use that to establish the result.
Assume that the time range is $[0, 1]$, and we discretize the time into $K$ intervals of equal length $\Delta t = 1/K$. The endpoints of the intervals are $0=t_0 < t_1 < \cdots < t_K = 1$.
For brevity, we will use $k$ instead of $t_k$ throughout the derivation.

Before we begin, we recall a few useful facts that we will use multiple times.

\begin{enumerate}
    \item Taylor expansion of log:
          \begin{align}\label{fact:taylor_log}
              \log(1+x) & = x - \frac{x^2}{2} + \frac{x^3}{3} - \cdots \nonumber \\
                        & = x - \frac{x^2}{2} + o(x^2)
          \end{align}
    \item For $a,b > 0$,
          \begin{align}
              \log(ab + o(a)) & = \log(a(b + o(1)))                 \nonumber   \\
                              & = \log(a) + \log(b + o(1))          \nonumber   \\
                              & = \log(a) + \log(b(1+ o(1)))        \nonumber   \\
                              & = \log(a) + \log(b) + \log(1+ o(1)) \nonumber   \\
                              & = \log(a) + \log(b) + o(1) \label{fact:log_add}
          \end{align}
    \item Bayes rule and the Markov property imply
          \begin{align}\label{fact:bayes_markov}
              q(x_k \mid x_{k+1}, x_0) & = \frac{ q(x_k \mid x_0)}{q(x_{k+1} \mid x_0)} q(x_{k+1} \mid x_k) \nonumber \\
                                         & = r(k, k+1) q(x_{k+1} \mid x_k),
          \end{align}
          where we denote the ratio $\frac{q(x_k \mid x_0)}{q(x_{k+1} \mid x_0)}$ with $r(k, k+1)$.
\end{enumerate}

With slight abuse of notation, we will denote the probability law of the noising process with $q$ and its time-reversal with $p$.
To begin, recall that the discrete-time ELBO \citep{hoDenoisingDiffusionProbabilistic2020} is given by
\begin{align*}
    \log p_\theta(x_0) \geq \underbrace{\mathop{\mathbb E}_{x_1\sim q_{x_1|x_0}} \log p_\theta(x_0|x_1)}_{-L_0} - \underbrace{\mathop{\mathbb E}_{x_{2:T}\sim q_{x_{2:T}|\,x_0}} \sum_{t=2}^T \KL\left[{q(x_{t-1}|x_{t}, x_0)}\,\Vert\, {p_\theta(x_{t-1}\,|\,x_{t})}\right]}_{L_{1:T-1}} - \underbrace{\KL[ q(x_T\,|\,x_0)\,\Vert\,p_\theta(x_T)]}_{L_T}
\end{align*}

We will focus on one term of $L_{1:T-1}$, denoting it by $L_k$; we will also drop $\theta$ from $p_\theta$ for brevity. Taking the expectation inside the summation for $L_{1:T-1}$ and writing one term, we get
\begin{align*}
    L_{k}(x_0) & = \E_{\substack{x_{k+1}\sim q_{k+1|0}}} \KL\left[q(x_k \mid x_{k+1}, x_0) \,\Vert\, p(x_k \mid x_{k+1})\right]\\
               & = \E_{\substack{x_{k+1}\sim q_{k+1|0}}} \mathcal -H[q(x_k \mid x_{k+1}, x_0)] + \textnormal{CE}[q(x_k \mid x_{k+1}, x_0)\,\Vert\, p(x_k \mid x_{k+1})]
\end{align*}

We will expand the entropy and the cross entropy terms by writing the conditional probabilities in terms of the forward and reverse rate matrices and making approximations knowing that we will ultimately apply the limit $\Delta t \to 0$, and $K\to \infty$. The approximations do not introduce any error but are a way to apply the limit throughout the derivation to manage the complexity of the expressions. The following are the expressions for the conditional probabilities in terms of the forward and reverse rate matrices:
\begin{align}
    q(x_k \mid x_{k+1}, x_0) & = r(k, k+1) \left(\delta_{x_k}(x_{k+1}) + R_k(x_k, x_{k+1}) \Delta t + o(\Delta t) \right) \nonumber \\
    p(x_k \mid x_{k+1})      & = \delta_{x_k}(x_{k+1}) + \hat R_{k+1}(x_{k+1}, x_k) \Delta t + o(\Delta t)
\end{align}
where we assume $r(k, k+1)$ exists and is finite, i.e., $q(x_{k+1} \mid x_0) > 0$. %
We will also suppress the indices involved in $r(k, k+1)$  and $\delta_{x_k}(x_{k+1})$ for brevity as they do not change in the following derivation.
So unless stated explicitly, $r=r(k, k+1)$ and $\delta=\delta_{x_k}(x_{k+1})$ and $\bar \delta = 1 - \delta_{x_k}(x_{k+1})$.

\begin{align*}
     & -\textnormal{CE}[q(x_k \mid x_{k+1}, x_0)\,\Vert\, p(x_k \mid x_{k+1})]                                                                                                 \\
     & = \sum_{x_k} q(x_k \mid x_{k+1}, x_0) \log p(x_k \mid x_{k+1}) \nonumber                                                                                                \\
     & = \sum_{x_k} r \left(\delta + R_k(x_k, x_{k+1}) \Delta t + o(\Delta t)\right) \log \left(\delta + \hat R_{k+1}(x_{k+1}, x_k) \Delta t + o(\Delta t) \right) \nonumber \\
\end{align*}

Note that
\begin{align}
    \log(\delta + \hat R_{k+1}(x_{k+1}, x_k) \Delta t + o(\Delta t)) & = \delta \log(1 + \hat R_{k+1}(x_{k+1}, x_k) \Delta t + o(\Delta t)) \nonumber                                                                             \\
                                                                     & ~~~~~~~~+ \bar \delta \log(\hat R_{k+1}(x_{k+1}, x_k) \Delta t + o(\Delta t))                                                     \nonumber                \\
                                                                     & = \delta (\hat R_{k+1}(x_{k+1}, x_k) \Delta t + o(\Delta t)) \nonumber                                                                                     \\
                                                                     & ~~~~~~~~+ \bar \delta \log(\hat R_{k+1}(x_{k+1}, x_k) \Delta t) + o(\Delta t) ~~~\text{ (by \ref{fact:taylor_log})} \nonumber                              \\
                                                                     & = \delta (\hat R_{k+1}(x_{k+1}, x_k) \Delta t + o(\Delta t)) \nonumber                                                                                     \\
                                                                     & ~~~~~~~~+ \bar \delta \log(\hat R_{k+1}(x_{k+1}, x_k) \Delta t) + o(\Delta t)  \nonumber                                                                   \\
                                                                     & = \delta (\hat R_{k+1}(x_{k+1}, x_k) \Delta t ) \nonumber                                                                                                  \\
                                                                     & ~~~~~~~~+ \bar \delta (\log(\hat R_{k+1}(x_{k+1}, x_k)) + \log(\Delta t) + o(1)) + o(\Delta t) ~~~\text{ (by \ref{fact:log_add})} \label{eq:expansion_log}
\end{align}

Therefore,
\begin{align}
     & -\textnormal{CE}[q(x_k \mid x_{k+1}, x_0)\,\Vert\, p(x_k \mid x_{k+1})]                                                                                                                     \nonumber \\
     & = \sum_{x_k} r \left\{\delta \hat R_{k+1}(x_{k+1}, x_k) \Delta t + \bar \delta \Delta t R_k(x_k, x_{k+1}) \log(\hat R_{k+1}(x_{k+1}, x_k)) + \Delta t \log(\Delta t) + o(\Delta t) \right\} \nonumber \\
     & = \sum_{x_k} r \left\{\delta \hat R_{k+1}(x_{k+1}, x_k) \Delta t + \bar \delta \Delta t R_k(x_k, x_{k+1}) \log(\hat R_{k+1}(x_{k+1}, x_k)) + o(\Delta t) \right\} \nonumber                           \\
     & = \Delta t \hat R_{k+1}(x_{k+1}, x_{k+1})  + \Delta t\sum_{x_k\not=x_{k+1}} r R_k(x_k, x_{k+1}) \log(\hat R_{k+1}(x_{k+1}, x_k)) + o(\Delta t)
\end{align}

Similarly,
\begin{align}
    \mathcal -H[q(x_k \mid x_{k+1}, x_0)] & = \sum_{x_k} r \left\{\delta  + R_k(x_k, x_{k+1}) \Delta t + o(\Delta t)  \right\} \log r \left\{\delta + R_k(x_k, x_{k+1}) \Delta t + o(\Delta t) \right\} \nonumber \\
                                          & = \sum_{x_k} r \left\{\delta R_k(x_k, x_{k+1}) \Delta t + \bar \delta \Delta t R_k(x_k, x_{k+1}) \log(R_k(x_k, x_{k+1})) + o(\Delta t) \right\} \nonumber             \\
                                          & ~~~~~~~~+ \sum_{x_k} \delta r \log r + r \Delta t R_k(x_k, x_{k+1}) \log r \nonumber                                                                                  \\
                                          & = \Delta t R_k(x_{k+1}, x_{k+1})  + \Delta t\sum_{x_k\not=x_{k+1}}  r R_k(x_k, x_{k+1}) \log(R_k(x_k, x_{k+1})) + o(\Delta t) \nonumber                               \\
                                          & ~~~~~~~~+ \Delta t \sum_{x_k\not=x_{k+1}} R_k(x_k, x_{k+1}) r \log r   ~~(\text{since } r\log r = 0 \text{ when } x_k=x_{k+1})
\end{align}

Putting the two together, we get
\begin{align}
     L_k(x_0) 
    & =\E_{\substack{x_{k+1}\Hide{\sim q_{k+1|0}}}} \sum_{x_k} r\Delta t \left\{ \underbrace{\left[\delta R_k(x_k, x_{k+1}) + \bar \delta  R_k(x_k, x_{k+1}) \log(R_k(x_k, x_{k+1})) + \bar \delta R_k(x_k, x_{k+1}) \log r  \right]}_{H \textnormal{ term}} \phantom{\hat R}\right. \nonumber \\
    & ~~~~~~~~~~~- \left. \underbrace{\left[ \delta \hat R_{k+1}(x_{k+1}, x_k) + \bar \delta   R_k(x_k, x_{k+1}) \log(\hat R_{k+1}(x_{k+1}, x_k))\right]}_{\operatorname{CE}\textnormal{term}} + \frac{o(\Delta t)}{\Delta t}  \right\}                   
\end{align}

Recall that $r = \frac{q(x_{k} \mid x_0)}{q(x_{k+1} \mid x_0)}$, and the expectation is over $x_{k+1}\sim q_{k+1|0}(\cdot \mid x_0)$. 
The entropy is constant with respect to $\hat R$. 

\begin{align*}
    \sum_k L_k(x_0) 
   & = \sum_k \Delta t \sum_{x_{k+1}} q(x_{k+1} \mid x_0) \sum_{x_k}  \frac{q(x_k \mid x_0)}{q(x_{k+1} \mid x_0)}\left\{ C(x_0, x_k, x_{k+1}) \phantom{\hat R}\right. \nonumber \\
   & ~~~~~~~~~~~- \left. \underbrace{\left[ \delta \hat R_{k+1}(x_{k+1}, x_k) + \bar \delta   R_k(x_k, x_{k+1}) \log(\hat R_{k+1}(x_{k+1}, x_k))\right]}_{\operatorname{CE}\textnormal{term}} + \frac{o(\Delta t)}{\Delta t}  \right\}                   
\end{align*}
Now recall that $\delta$ and $\bar \delta$ are shorthand for $\delta_{x_t}(y)$ and $1 - \delta_{x_t}(y)$, respectively, and therefore

\begin{align}\label{eq:delta_hat_R}
\sum_{y}\frac{q(y|x_0)}{q(x_t|x_0)}\delta_{x_t}(y) \hat R_t(x_t, y) = \hat R_t(x_t, x_t) = - \sum_{y\not=x_t} \hat R_t(x_t, y).
\end{align}

As $\Delta t \to 0, K\to \infty$, we get $x_{k+1} \to x_{t_{k+1}} = x_t$ for $t\in [0, 1]$, and the summation becomes an integral (since the Markov chain is assumed to be regular, the limit is well defined). Denoting $\lim_{K\to \infty} \sum_k L_k(x_0)$ as $\mathcal L(x_0)$ and using \Cref{eq:delta_hat_R}, we get
\begin{align*}
    \mathcal L(x_0) &= \int_0^1 \sum_{x_t} q(x_t\,|\, x_0) \sum_{y}\frac{q(y\,|\, x_0)}{q(x_t\,|\, x_0)} \left[ \delta \hat R_{t}(x_t, y) + \bar \delta   R_t(y, x_t) \log(\hat R_{t}(x_{t}, y))\right] dt  + C(x_0)\nonumber \\
    &= \E_{t, x_t} \sum_{y\not=x_t}\left[ - \hat R_{t}(x_t, y) +  \frac{q(y\,|\, x_0)}{q(x_t\,|\, x_0)}  R_t(y, x_t) \log(\hat R_{t}(x_{t}, y))\right] dt + C(x_0)\nonumber 
\end{align*}

\end{proof}

%% file: drafts/aistats/sections/appendices/direct_parameterization.tex
\begin{restatable}[Independent Noising]{corollary}{CorELBOSecondInd}\label{cor:elbo_independent_insertion_per_position}
    A (biased) estimate of the upper bound on the negative log-likelihood in \Cref{prop:base_elbo} for the independent per-position parameterization of $\hat R_t^\theta$, when the noising process is per-position deletion (\Cref{prop:transition_independent}), is given by
    \begin{align*}
        &\mathcal L(\theta) = \mathcal L_{\textnormal{len}}(\theta) + \mathcal L_{\textnormal{token}}(\theta),~\text{where} \\
        &\mathcal L^{\textnormal{len}}(\theta) = \E\left[\sum_{i=1}^{m+1} \hat \lambda_t^\theta(\x_0[\bm b], i) - \gamma_t^{\text{ind}} |s_i(\bm b)| \log \hat \lambda_t^\theta(\x_0[\bm b], i)\right],  \\
        &\mathcal L^{\textnormal{token}}(\theta) = -\E\left[\gamma_t^{\text{ind}} \sum_{i=1}^{m+1} \sum_{k \in s_i(\bm b)} \log \hat p_{t}^{\theta}(x_0^k \mid \x_0[\bm b], i)\right].
    \end{align*}
    Here, the expectation is over $\x_0 \sim p_{\textnormal{data}}$, $t\sim \textnormal{Uniform}[0, T]$, $m \sim \textnormal{Binomial}(n, \rho_t)$, $\bm b \sim \textnormal{Uniform}(\sB_{n, m})$, $\gamma_t^{\text{ind}} = \frac{\sigma_t\,\rho_t}{1-\rho_t}$ with $\rho_t = \exp(-\bar \sigma_{0,t})$,
    $\x_0=(n, \seqx_0)$, and $s_i(\bm b)$ is the set of indices of 0s in $\bm b$ that fall between the $(i-1)$-th and $i$-th 1s.
\end{restatable}
\begin{proof}
For $\x = (m, \seqx)$, write
\begin{align*}
\hat R_t^{\theta}(\x, \y) = \sum_{i=1}^{m+1} \sum_{w \in \vocab} \hat \lambda_t^{\theta}(\x, i)\,\hat p_t^{\theta}( w \mid \x, i)~ \delta_{\ins(\x, i, w)}(\y).
\end{align*}
Fix $\x_0=(n,\seqx_0)$ and $\x=\x_0[\bm b]$. As in the proof of \Cref{cor:elbo}, we plug this factorization into \Cref{prop:base_elbo} and separate the contribution into the negative rate term and the predecessor-weighted log term.

The negative term is exact:
\begin{align*}
\sum_{\y \neq \x}\hat R_t^\theta(\x,\y)
= \sum_{i=1}^{m+1}\sum_{w\in\vocab}\hat \lambda_t^\theta(\x,i)\hat p_t^\theta(w\mid \x,i)
= \sum_{i=1}^{m+1}\hat \lambda_t^\theta(\x,i).
\end{align*}
For the log term, if $\y=\ins(\x,i,x_0^k)$ for some $k\in s_i(\bm b)$, then \Cref{prop:transition_independent} gives
\begin{align*}
\frac{p_{t|0}(\y\mid \x_0)}{p_{t|0}(\x\mid \x_0)}
= \frac{\rho_t^{m+1}(1-\rho_t)^{n-m-1}}{\rho_t^m(1-\rho_t)^{n-m}}
= \frac{\rho_t}{1-\rho_t},
\end{align*}
and the forward deletion rate from $\y$ to $\x$ is $\sigma_t$. Hence every missing original index contributes the same coefficient
\begin{align*}
\gamma_t^{\textnormal{ind}} = \frac{\sigma_t\,\rho_t}{1-\rho_t}.
\end{align*}

As in \Cref{cor:elbo}, we now use the same approximation for the log term: for each gap $i$, we pair the deleted indices $k\in s_i(\bm b)$ with the factorized reverse rate and replace
\begin{align*}
\log \hat R_t^\theta(\x_0[\bm b], \ins(\x_0[\bm b], i, x_0^k))
\end{align*}
by
\begin{align*}
\log \hat \lambda_t^\theta(\x_0[\bm b], i) + \log \hat p_t^\theta(x_0^k \mid \x_0[\bm b], i).
\end{align*}

With this approximation, the contribution to the upper bound on the negative log-likelihood becomes
\begin{align*}
\E\Bigg[ \sum_{i=1}^{m+1} \hat \lambda_t^\theta(\x_0[\bm b], i)
- \gamma_t^{\textnormal{ind}} \sum_{i=1}^{m+1}|s_i(\bm b)| \log \hat \lambda_t^\theta(\x_0[\bm b], i) \\
- \gamma_t^{\textnormal{ind}} \sum_{i=1}^{m+1}\sum_{k\in s_i(\bm b)} \log \hat p_t^\theta(x_0^k \mid \x_0[\bm b], i)\Bigg],
\end{align*}
where the expectation is over $\x_0 \sim p_{\textnormal{data}}$, $t\sim \textnormal{Uniform}[0,T]$, $m\sim \textnormal{Binomial}(n,\rho_t)$, and $\bm b\sim \textnormal{Uniform}(\sB_{n,m})$. Grouping the first two terms as $\mathcal L_{\textnormal{len}}(\theta)$ and the last term as $\mathcal L_{\textnormal{token}}(\theta)$ gives the stated result.
\end{proof}

%% file: drafts/aistats/sections/appendices/extended_related_work.tex
\section{Extended Discussion on Related Work}\label{app:extended_related_work}

\xhdr{Relation to discrete flow matching and stochastic interpolants.}
Our work is closely related in spirit to recent discrete flow matching and stochastic interpolant approaches for sequence generation~\citep{campbellGenerativeFlowsDiscrete2024,albergo2023stochasticinterpolantsunifyingframework}. In those frameworks, one typically specifies data-conditional target rates $R_t(x_t,\cdot \mid x_0)$ directly and trains a model $\hat R_t^\theta(x_t,\cdot)$ to match them, often using a Bregman divergence. 
In contrast, our starting point is an explicit diffusion-style noising process: a deletion CTMC on variable-length sequences whose reverse-time dynamics define the generative insertion process. Conceptually, the target-rate construction in flow matching plays a role analogous to the choice of noising process in diffusion, but the modeling perspective is different: in our case the reverse process is obtained from a concrete forward deletion mechanism, rather than being specified directly.

\xhdr{Comparison with Edit Flows.}
\citet{havasi2025edit} also study continuous-time generation on discrete sequences, but they do so through conditional discrete flow matching with \emph{edit} operations. The high-level similarity is that both approaches use Markovian sequence evolution in continuous time and both can be trained with closely related rate-matching objectives. The main difference is that our construction is tailored to insertion-only generation. We begin with a deletion process that acts directly on complete sequences and then derive the reverse insertion model, whereas Edit Flows directly parameterizes a generative chain over a broader class of edit operations: insertions, deletions, and substitutions.
The presence of deletion operations in the generative process necessitates that the noising process insert spurious tokens thereby making the explicit expectation over the alignments completely intractable.
Under an insertion-only restriction our formulation provides a way to perform Rao-Blackwellized estimation of the ELBO objective using  dynamic programming as stated in \cref{remark:rao_blackwellization}.
If the Edit Flows operations were restricted to only insertions, and the alignments were sampled as we currently do in our implementation, the Bregman divergence for insertion-only rate matching, after ignoring the constant terms, would produce a training objective similar to the one given in our \cref{cor:elbo_independent_insertion_per_position}.

\xhdr{Comparison with FlexMDM.}
FlexMDM~\citep{kim2025anyorder} is a concurrent work that also generates sequences by growing a partial sequence over time using insertions. The main distinction between DILM and FlexMDM lies in the choice of \emph{what} is inserted and \emph{how} the forward process is defined. 
In our model, a generative step inserts a vocabulary token directly into a gap. In FlexMDM, a generative step inserts a fresh \texttt{MASK} token into a gap and separately allows already-present \texttt{MASK} tokens to be unmasked. Thus both methods are insertion-based, but ours chooses token identity at insertion time, whereas FlexMDM defers that decision to a later unmasking step.

%% file: drafts/aistats/sections/appendices/implementation.tex
\section{Implementation Details}\label{app:implementation_details}
\subsection{Choosing the Noise Schedule}\label{app:choosing_noise_schedule}
Unlike fixed-length MDMs, where the noise schedule is not as critical~\citep{zheng2025masked,ouYourAbsorbingDiscrete2024}, picking the right noise schedule is important for variable-length generation. We want the generator to generalize across different sequence lengths, and therefore we need to pick a noise schedule that provides sufficient coverage of all final lengths in the noised sequences.

\subsubsection{Joint Noising}
\begin{wrapfigure}{r}{0.45\textwidth}
    \centering
    \vspace{-2mm}
    \includegraphics[width=0.35\textwidth]{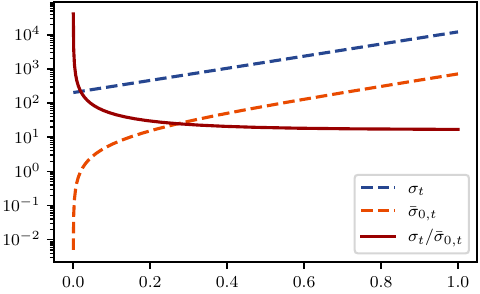}
    \vspace{-2mm}
    \caption{Log-linear noise schedule.}
    \label{fig:loglinear_noise_shape}
    \vspace{-8mm}
\end{wrapfigure}
We use the \emph{log-linear} schedule for joint noising defined as
\begin{align*}%
    \sigma(t) = \sigma_{\min}\left(\frac{\sigma_{\max}}{\sigma_{\min}}\right)^{t},
\end{align*}
where $0<\sigma_{\min}<\sigma_{\max}$ are hyperparameters (in particular $\sigma(0)=\sigma_{\min}$ and $\sigma(1)=\sigma_{\max}$).
The corresponding integrated rate from \cref{prop:transition_probability} is
\begin{align*}%
    \bar{\sigma}_{0,t} = \int_{0}^{t} \sigma(u)\,\dd{u}
    = \frac{\sigma_{\min}}{\ln(\sigma_{\max}/\sigma_{\min})}\left(\left(\frac{\sigma_{\max}}{\sigma_{\min}}\right)^{t}-1\right).
\end{align*}
\begin{figure}[t]
    \centering
    \begin{subfigure}[t]{0.48\linewidth}
        \centering
        \includegraphics[width=\linewidth]{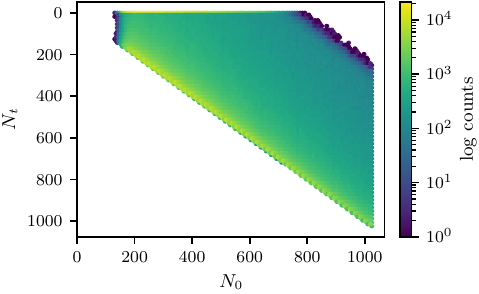}
        \caption{Joint Poisson noising.}
        \label{fig:nt_vs_n0_owt}
    \end{subfigure}\hfill
    \begin{subfigure}[t]{0.48\linewidth}
        \centering
        \includegraphics[width=\linewidth]{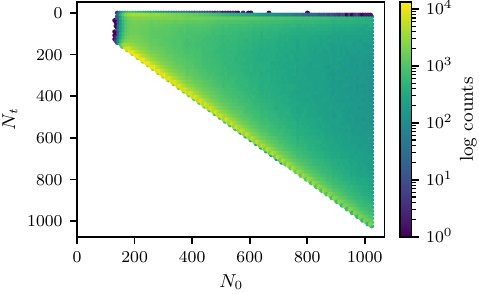}
        \caption{Independent noising.}
        \label{fig:nt_vs_n0_owt_indep}
    \end{subfigure}
    \caption{Empirical joint distributions of initial length $\cN_0$ and noised length $\cN_t$ on OpenWebText training documents.
    For each sequence we draw $t$ from the same uniform distribution on $[\varepsilon,1]$ used in training ($\varepsilon=10^{-4}$).
    In \Cref{fig:nt_vs_n0_owt}, we sample $D \sim \mathrm{Poisson}(\bar{\sigma}_{0,t})$ with the log-linear schedule $(\sigma_{\min}, \sigma_{\max})=(50,3000)$ and set $\cN_t=\max(\cN_0-D,0)$.
    In \Cref{fig:nt_vs_n0_owt_indep}, we sample $D \sim \Binomial(\cN_0, 1-\exp(-\bar{\sigma}_{0,t}))$ with $(\sigma_{\min}, \sigma_{\max})=(0.1,20)$ and set $\cN_t=\cN_0-D$.
    Both panels show hexagonal-bin densities with log-scaled counts (darker is more mass).
    Joint noising yields additive shrinkage, with visible mass near $\cN_t=0$ for short inputs, whereas independent noising is closer to multiplicative shrinkage, so the typical noised length tracks the original length more closely.}
    \label{fig:nt_vs_n0_owt_compare}
\end{figure}
During training we sample $t$ uniformly from $[\varepsilon,1]$, with $\varepsilon$ small (e.g.\ $\varepsilon=10^{-4}$ in the configuration).
The log-linear schedule has the desirable property that the process slows down near $t=0$. We did not explore other schedules in our experiments.

Under joint Poisson noising (\cref{alg:sample_noise}), let $\cN_t$ denote the length at time $t$ given $\cN_0=n$.
Equivalently, sample $D \sim \mathrm{Poisson}(\bar{\sigma}_{0,t})$ for the number of positions to drop and set $\cN_t = \max(n-D,0)$ (after the $\min$ with $n$ in the algorithm).
Then for $m \in \{0,1,\ldots,n\}$,
\begin{align*}
    &\sP\!\left(\cN_t = m \mid \cN_0 = n\right)\\
    &=
    \begin{cases}
        \displaystyle 1 - \sum_{k=0}^{n-1} \frac{e^{-\bar{\sigma}_{0,t}}\,\bar{\sigma}_{0,t}^{\,k}}{k!}, & m = 0, \\[0.9em]
        \displaystyle \frac{e^{-\bar{\sigma}_{0,t}}\,\bar{\sigma}_{0,t}^{\,n-m}}{(n-m)!}, & m \in \{1,\ldots,n\}.
    \end{cases}
\end{align*}

Its conditional expectation is
\begin{align*}
    \E[\cN_t \mid \cN_0=n]
    &= a_n(\bar{\sigma}_{0,t})\,n - b_n(\bar{\sigma}_{0,t})\,\bar{\sigma}_{0,t},
\end{align*}
where
\begin{align*}
    a_n(\lambda) &= \sum_{k=0}^{n-1} \frac{e^{-\lambda}\lambda^k}{k!},
    \qquad
    b_n(\lambda) = \sum_{k=0}^{n-2} \frac{e^{-\lambda}\lambda^k}{k!}.
\end{align*}
Thus joint noising is additive in the sense that the expected length is a linear combination of the original length $n$ and the integrated rate $\bar{\sigma}_{0,t}$, with truncation-dependent coefficients.
When truncation is negligible, $a_n(\bar{\sigma}_{0,t}) \approx b_n(\bar{\sigma}_{0,t}) \approx 1$, recovering the heuristic scale $n-\bar{\sigma}_{0,t}$.

\Cref{fig:nt_vs_n0_owt} illustrates how these choices interact with the heavy-tailed length distribution of OpenWebText: the model trains on a wide range of $(\cN_0,\cN_t)$ pairs rather than a single fixed shrinkage factor, which motivates tuning the schedule so that typical noised lengths still cover the range needed at generation time.

\subsubsection{Independent Noising}
Under independent noising, each position survives until time $t$ with probability
\begin{align*}
    \rho_{0,t} = \exp(-\bar{\sigma}_{0,t}),
\end{align*}
so \Cref{alg:sample_noise} samples
\begin{align*}
    D \sim \Binomial\!\left(n, 1-\rho_{0,t}\right),
\end{align*}
and returns $\cN_t=n-D$.
Equivalently, conditional on $\cN_0=n$, the noised length has marginal
\begin{align*}\label{eq:independent_length_marginal}
    \sP\!\left(\cN_t = m \mid \cN_0 = n\right)
    = \binom{n}{m}\rho_{0,t}^{\,m}\left(1-\rho_{0,t}\right)^{n-m},
    \qquad m \in \{0,1,\ldots,n\}.
\end{align*}
Hence
\begin{align*}
    \E[\cN_t \mid \cN_0=n] = n\rho_{0,t},
\end{align*}
so independent noising induces multiplicative shrinkage of sequence length.
This contrasts with joint Poisson noising, where the number of deletions is additive and, before truncation at zero, independent of $n$.
\Cref{fig:nt_vs_n0_owt_indep} shows the resulting behavior on OpenWebText: the noised lengths remain broadly distributed, but their typical scale tracks the original length more closely than in the joint-noising case.

\subsection{Computing the diffusion coefficients}\label{app:compute_ratio}
Computing the diffusion coefficient $\gamma_t^{\textnormal{ind}}$ in the independent noising case is straightforward, but this is not the case for the joint noising case. Here we provide details on how to numerically compute the diffusion coefficient for the joint noising case.
For the joint noising case, with $t>s$, and $m>0$, from \eqref{eq:dropping_process_transition_probability} we have
\begin{align*}
    \frac{\sP(\cN_t=m+1 \mid \cN_s=n)}{\sP(\cN_t=m \mid \cN_s=n)} & = \frac{ e^{-\bar \sigma_{s, t}} (\bar \sigma_{s, t})^{n-m-1}}{(n-m-1)!} \frac{(n-m)!}{ e^{-\bar \sigma_{s, t}} (\bar \sigma_{s, t})^{n-m}} \\
    & = \frac{n-m}{\bar \sigma_{s, t}}.
\end{align*}
For the case of $m=0$, we need to compute the tail of the exponential for the denominator.
\begin{align*}
    \frac{\sP(\cN_t=m+1 \mid \cN_s=n)}{\sP(\cN_t=m \mid \cN_s=n)} & =  \frac{\sP(\cN_t=1 \mid \cN_s=n)}{\sP(\cN_t=0 \mid \cN_s=n)} \\
    &= \frac{e^{-\bar \sigma_{s, t}} (\bar \sigma_{s, t})^{n-1}}{(n-1)!} \frac{1}{1 - \sum_{k=0}^{n-1} \frac{ e^{-\bar \sigma_{s, t}} (\bar \sigma_{s, t})^{k}}{k!}}\\
    &= \frac{e^{-\bar \sigma_{s, t}} (\bar \sigma_{s, t})^{n-1}}{(n-1)!} \frac{1}{\Phi(n, \bar \sigma_{s, t})}\\
    &= \frac{n}{\bar \sigma_{s, t}}\frac{1}{S(n, \bar \sigma_{s, t})}
\end{align*}
where $S(n, \bar \sigma_{s, t}) = \frac{n!\,e^{\bar \sigma_{s, t}}}{(\bar \sigma_{s, t})^n} \Phi(n, \bar \sigma_{s, t})$ and $\Phi(n, \bar \sigma_{s, t}) = (1 - \frac{\Gamma(n, \bar \sigma_{s, t})}{\Gamma(n)})$ is the regularized lower incomplete gamma function.
We vectorize the computation in PyTorch in a numerically stable manner using Algorithm~\ref{alg:compute_ratio}, where $\text{lgamma}$ and $\text{gammainc}$ are \texttt{torch.special.lgamma} and \texttt{torch.special.gammainc}, respectively.

\begin{algorithm}
    \caption{Computing $\frac{\sP(\cN_t=m+1 \mid \cN_s=n)}{\sP(\cN_t=m \mid \cN_s=n)}$}
    \label{alg:compute_ratio}
    \begin{algorithmic}[1]
        \REQUIRE $n,m$,  $\bar \sigma_{s, t}$
        \STATE $u \gets \frac{n-m}{\bar \sigma_{s, t}}$
        \IF{$m=0$}
            \STATE $S \gets \exp \left[{\text{lgamma}(n+1) + \,{\bar \sigma_{s, t}} - n\log \bar \sigma_{s, t}}+ \text{gammainc}(n, \bar \sigma_{s, t})\right]$
            \STATE $u \gets \frac{u}{S}$
        \ENDIF
        \RETURN $u$
    \end{algorithmic}
\end{algorithm}

For large values of $n$, the regularized lower incomplete gamma function can quickly approach values less than $10^{-50}$, which leads to numerical instability.
Another alternative is to use the confluent hypergeometric function.
\begin{align*}
    S(n, \bar \sigma) &= \frac{n!\,e^{\bar \sigma}}{\bar \sigma^n} \Phi(n, \bar \sigma) \\
    &= \frac{n!\,e^{\bar \sigma}}{\bar \sigma^n} \left(1 - \sum_{k=0}^{n-1} \frac{ e^{-\bar \sigma} \bar \sigma^{k}}{k!}\right)\\
    &=\frac{n!}{\bar \sigma^n} \sum_{k=n}^{\infty} \frac{\bar \sigma^k}{k!}\\
    &= \frac{n!}{\bar \sigma^n} \sum_{k=0}^{\infty} \frac{\bar \sigma^{k+n}}{(k+n)!}\\
    &= \sum_{k=0}^{\infty} \frac{n!}{(n+k)!} \bar \sigma^{k}\\
    &= \text{hyp1f1}(1, n+1; \bar \sigma)
\end{align*}

Since \texttt{hyp1f1} is not available in PyTorch, we can simply use the series expansion.

\begin{longlisting}[label={lst:hyp1f1}]{Computing \texttt{hyp1f1} using series expansion}
def hyp1f1_1_nplus1_vec(x, n, K=500):
    # x: scalar tensor, n: (batch,) tensor
    device = x.device
    n = n.unsqueeze(1).to(torch.float64)  # shape (batch, 1)
    x = x.unsqueeze(1).to(torch.float64)  # shape (batch, 1)

    # create matrix of denominators of shape (*batch, K), where *batch is the leading dimensions of x
    ks = torch.arange(K, dtype=x.dtype, device=device).unsqueeze(
        0
    )  # k=0..K-1, shape (1, K)
    den = n + 1 + ks  # shape (batch, K)

    # factors = x / (n+1+k)
    factors = x / den  # shape (batch, K)

    # compute cumulative product along k to get T_k/T_0
    cumfac = torch.cumprod(factors, dim=-1)  # shape (batch, K)

    # prepend T_0=1 to align
    T = torch.cat([torch.ones_like(n), cumfac], dim=-1)  # (batch, K+1)

    # sum over k
    return T.sum(dim=-1)  # shape (batch,)
\end{longlisting}

\subsection{Efficient Training}
Efficient training of insertion language models requires careful consideration of engineering details. 
Here we discuss some low-level details that are essential for making the training practical.

\xhdr{Multi-worker collation on CPU using sparse tensors}
Computing the token loss requires, for each gap in the corrupted sequence, identifying all original tokens that were removed from that gap. Equivalently, each gap must be paired with a vocabulary-sized multi-hot target vector that has 1s exactly on the tokens belonging to that gap. 
Since the number and identity of removed tokens vary irregularly across gaps and examples, these targets are most naturally assembled during collation by scanning each example using a for-loop and writing only the nonzero entries into a sparse representation.

In practice, this work is best done by CPU dataloader workers in parallel, while the GPU is reserved for the forward and backward passes. It is also important to keep these targets sparse until they have been transferred to the GPU: materializing the dense tensors of shape \texttt{(batch, seq\_len, vocab)} on CPU leads to prohibitively high CPU memory usage in multi-worker settings and wastes CPU-GPU transfer bandwidth on tensors that are almost entirely zero, slowing down training. 
We therefore construct sparse targets during CPU collation and expand them only on the GPU when evaluating the loss. 

With these optimizations, the training of insertion language models is as efficient as MDMs and ARMs in terms of examples per second.

\subsection{Sampling}\label{app:sampling}

\subsubsection{Multi-token insertions}\label{app:lst:multi_token_step}
Figure~\ref{fig:insertion-cumsum} visualizes how multiple insertions can be performed in a single tensorized step without for-loops using cumulative sums to determine shifted positions of existing tokens and the destinations of newly inserted tokens. The listing that follows shows a minimal PyTorch implementation of this core multi-token insertion update.

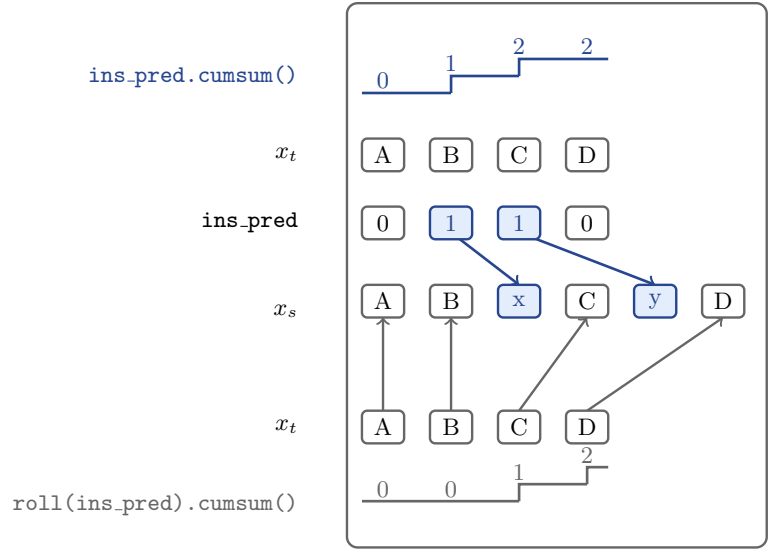
\begin{figure}[H]
    \centering
    \begin{tikzpicture}[
        x=0.9cm,
        y=0.9cm,
        every node/.style={font=\footnotesize},
        stepblue/.style={draw=cics_blue, line width=1.0pt},
        stepgray/.style={draw=cics_gray, line width=1.0pt},
        arrowblue/.style={->, draw=cics_blue, line width=1.0pt},
        arrowgray/.style={->, draw=cics_gray, line width=0.9pt},
        boxgray/.style={draw=cics_gray, fill=white, rounded corners=2pt, line width=0.9pt},
        boxblue/.style={draw=cics_blue, fill=cics_bglight_blue, rounded corners=2pt, line width=0.9pt},
        boxblueoutline/.style={draw=cics_blue, fill=white, rounded corners=2pt, line width=0.9pt}
    ]
        \node[anchor=east,text=cics_blue] at (0.2,3.45) {$\texttt{ins\_pred.cumsum()}$};
        \node[anchor=east] at (0.2,2.3) {$x_t$};
        \node[anchor=east] at (0.2,1.3) {$\texttt{ins\_pred}$};
        \node[anchor=east] at (0.2,0.0) {$x_s$};
        \node[anchor=east] at (0.2,-1.7) {$x_t$};
        \node[anchor=east,text=cics_gray] at (0.2,-2.85) {$\texttt{roll(ins\_pred).cumsum()}$};

        \draw[stepblue] (1.00,3.20) -- (2.31,3.20);
        \draw[stepblue] (2.31,3.20) -- (2.31,3.45) -- (3.31,3.45);
        \draw[stepblue] (3.31,3.45) -- (3.31,3.70) -- (4.62,3.70);
        \foreach \x/\lbl/\y in {1.31/0/3.38,2.31/1/3.63,3.31/2/3.88,4.31/2/3.88} {
            \node[text=cics_blue] at (\x,\y) {\lbl};
        }

        \draw[stepgray] (1.00,-2.80) -- (3.31,-2.80);
        \draw[stepgray] (3.31,-2.80) -- (3.31,-2.55) -- (4.31,-2.55);
        \draw[stepgray] (4.31,-2.55) -- (4.31,-2.30) -- (4.62,-2.30);
        \foreach \x/\lbl/\y in {1.31/0/-2.62,2.31/0/-2.62,3.31/1/-2.37,4.31/2/-2.12} {
            \node[text=cics_gray] at (\x,\y) {\lbl};
        }

        \draw[draw=cics_gray, line width=0.9pt, rounded corners=4pt] (0.78,-3.48) rectangle (6.95,4.52);

        \draw[arrowblue] (2.31,1.15) -- (3.31,0.38);
        \draw[arrowblue] (3.31,1.15) -- (5.31,0.38);

        \draw[arrowgray] (1.31,-1.47) -- (1.31,-0.10);
        \draw[arrowgray] (2.31,-1.47) -- (2.31,-0.10);
        \draw[arrowgray] (3.31,-1.47) -- (4.31,-0.10);
        \draw[arrowgray] (4.31,-1.47) -- (6.31,-0.10);

        \foreach \i/\lbl in {1/A,2/B,3/C,4/D} {
            \filldraw[boxgray] (\i,2.05) rectangle +(0.62,0.48);
            \node at (\i+0.31,2.29) {\lbl};
        }

        \foreach \i/\lbl/\style/\col in {1/0/boxgray/black,2/1/boxblue/cics_blue,3/1/boxblue/cics_blue,4/0/boxgray/black} {
            \filldraw[\style] (\i,1.05) rectangle +(0.62,0.48);
            \node[text=\col] at (\i+0.31,1.29) {\lbl};
        }

        \foreach \i/\lbl/\style/\col in {1/A/boxgray/black,2/B/boxgray/black,3/x/boxblue/cics_blue,4/C/boxgray/black,5/y/boxblue/cics_blue,6/D/boxgray/black} {
            \filldraw[\style] (\i,-0.10) rectangle +(0.62,0.48);
            \node[text=\col] at (\i+0.31,0.14) {\lbl};
        }

        \foreach \i/\lbl in {1/A,2/B,3/C,4/D} {
            \filldraw[boxgray] (\i,-1.95) rectangle +(0.62,0.48);
            \node at (\i+0.31,-1.71) {\lbl};
        }
    \end{tikzpicture}
    \caption{Illustration of the insertion update. The top panel isolates the insertion slots: $\texttt{ins\_pred.cumsum()}$ indexes the new positions that receive the inserted tokens. The bottom panel isolates the movement of existing tokens: $\texttt{roll(ins\_pred).cumsum()}$ counts earlier insertions and therefore determines how far each original token shifts to the right.}
    \label{fig:insertion-cumsum}
\end{figure}

\begin{longlisting}[label={lst:multi_token_step}]{Core multi-token insertion update in PyTorch}
    def insert_sampled_tokens(
        x_t: torch.Tensor,        # (batch, seq_len)
        ins_pred: torch.Tensor,   # bool mask: insert after this position?
        ins_tokens: torch.Tensor, # sampled token for each active insertion
        pad_token_id: int = 0,
    ) -> torch.Tensor:
        batch_size, seq_len = x_t.shape
        x_s = torch.full_like(x_t, pad_token_id)

        inc_positions = torch.arange(seq_len, device=x_t.device)
        inc_positions = inc_positions.unsqueeze(0).expand(batch_size, -1)

        # Existing tokens shift right by the number of insertions before them.
        existing_tokens_new_positions = (
            inc_positions + ins_pred.roll(shifts=1, dims=-1).int().cumsum(dim=-1)
        ).clamp(max=seq_len - 1)

        # Inserted tokens go in the newly opened slots.
        inserted_tokens_new_positions = (
            inc_positions + ins_pred.int().cumsum(dim=-1)
        ).clamp(max=seq_len - 1)

        batch_index = torch.arange(batch_size, device=x_t.device)
        batch_index = batch_index.unsqueeze(-1).expand(-1, seq_len)

        x_s.scatter_(dim=-1, index=existing_tokens_new_positions, src=x_t)
        x_s[batch_index[ins_pred], inserted_tokens_new_positions[ins_pred]] = (
            ins_tokens[ins_pred]
        )
        return x_s
    \end{longlisting}

\subsection{Data preprocessing}
\paragraph{Graph traversal}
The data for the graph traversal task is obtained from \citet{patel2025insertion}.
\paragraph{Language modeling}
Following \citet{patel2025insertion}, we train one model on LM1B tokenized with BERT tokenizer, with sequences truncated to a maximum length of 128 tokens.
The other model is trained on OpenWebText-1024 split, which is the same as OpenWebText but tokenized using GPT-2 tokenizer and keeping sequences of length 1024 tokens or less. This preserves about 80\% of the corpus.
We use the xLM library~\citep{patel-etal-2026-xlm} for the training and evaluation boilerplate code.

\subsubsection{Hyperparameters}
\Cref{tab:appendix-hyperparams} lists the key training hyperparameters.

\begin{table}[H]
\centering
\footnotesize
\setlength{\tabcolsep}{2pt}
\begin{tabular*}{\linewidth}{@{\extracolsep{\fill}}lccccc@{}}
\toprule
 & LM1B & OpenWebText-1024 & STAR (easy) & STAR (medium) & STAR (hard) \\
\midrule
Per-device batch size & 128 & 32 & 64 & 64 & 64 \\
Num. GPUs & 8 & 8 & 1 & 1 & 1 \\
Global batch size & 512 & 512 & 64 & 64 & 64 \\
Output length & 128 & 1024 & 14 & 16 & 28 \\
Input length  & --- & --- & 28 & 36 & 116 \\
Noise schedule $(\sigma_{\max},\sigma_{\min},\varepsilon)$ & 500, 5, $10^{-4}$ & 3000, 50, $10^{-4}$ & 30, 3, $10^{-3}$ & 30, 3, $10^{-3}$ & 50, 5, $10^{-3}$ \\
Training steps & $1000k$ & $300k$ & $80k$ & $80k$ & $80k$ \\
Learning rate & $5\times 10^{-5}$ & $10^{-5}$ & $10^{-4}$ & $10^{-4}$ & $10^{-4}$ \\
Weight decay & 0.03 & 0.03 & 0.01 & 0.01 & 0.01 \\
Adam $\varepsilon$ & $5\times 10^{-6}$ & $5\times 10^{-6}$ & --- & --- & --- \\
LR schedule & Cosine + min LR & Cosine + min LR & Constant & Constant & Constant \\
Minimum LR (cosine floor) & $10^{-6}$ & $10^{-6}$ & --- & --- & --- \\
LR warmup steps & 2000 & 2000 & 500 & 500 & 500 \\
Cosine cycles & 4.0 & 4.0 & --- & --- & --- \\
\bottomrule
\end{tabular*}
\caption{Training hyperparameters for all models (noise schedule only applicable to {\idlm{}} and {\idlmm{}}).}
\label{tab:appendix-hyperparams}
\end{table}
We picked reasonable defaults for the hyperparameters, and it is possible that the results could be improved with more careful hyperparameter tuning.

\subsubsection{Computing infrastructure}
For the language modeling experiments on LM1B and OpenWebText-1024, we trained all the models on 8 A100-80GB GPUs with `torch.compile` using the hyperparameters in \Cref{tab:appendix-hyperparams}. The training for 1 million steps on LM1B took approximately 28--30 hours, and on OpenWebText-1024 training for 300 thousand steps took approximately 72--80 hours.
For the experiments on star graphs, we use a single A100-80GB GPU, with the training taking 2--3 hours.

\section{Unconditional Generation Examples}\label{app:unconditional-generation-examples}
\subsection{{\idlm{}} on OpenWebText}
The following are unconditional samples from \jointhighlight{\idlm{}} on OpenWebText trained for 300,000 steps.
\begin{tcolorbox}[breakable, colback=cics_bglight_blue]
\ttfamily
A small woman said on Tuesday that a very great, very bad man had been looking for him with her husband, who has been playing with his son, he said.

"He's not really going to have done his whole job in this case, and he has always said that, but he's not talking about him," he said.

"He's got the great man working up with him, and getting him back out there, and he is not talking about himself, because he is on the edge of him, even when you have to call him up, and it's all the right thing.'

"And he's not sitting on the good man, which is that he has been a really good son, but that his son has always come out of him, and that he is his son, his son, and his wife."

"I think that he is going to have him, right now, that he's not like me, because he's not looking for me.

"I think that he's not looking at him, that they are just going to play with him, and that they are not looking for his parents, and when I'm not running up with him, this is a lot of a so-called good man. He's going to be a really good kid, and because they're not getting out at me, he's a very good man, and it is really great that he is going to have it, but I just feel like being a very bad man."

"He's just right here in that room, because he has never been a really good guy, and as he goes on, he's just going to be getting out behind his son," he said, looking at a very bad man, and saying that he's always been told that he is only running up on his side," he said.

"He said that, "I will all see what he is going into being out there."

"He's not really going to see him, and he is not going to be the same guy, and he wouldn't be going to have to see him.

"I'd like to know that, and he's going to be out there, and he's not saying that, and he really knows that, and for sure, you're not in that room."
\end{tcolorbox}
\begin{tcolorbox}[breakable, colback=cics_bglight_blue]
\ttfamily
The next day, I don't have signed up for the first time now. You are "signing it up in here, and you know that there is not anything going on from it, and that the company that's ever caught up with me being a great group," or any other part of that episode (which it just fades around) seems to be going off to make my view of this event, and some of you have been living in that episode, even after bringing it up here, in some time.

"The only way we are going to be doing it is that you are going to go down right here, and we know that you are not really stepping up with them."

Because I don't think that it's just going on in all the good news, and as if it is actually going to go down here? "And if you will go back to thinking about it," it's always going to find it up again.
\end{tcolorbox}

\subsection{{\idlm{}} on LM1B}

\begin{tcolorbox}[colback=cics_bglight_blue]
\textbf{Dataset: LM1B; NLL = 2.36; Entropy = 3.42}\\
\ttfamily
the first quarter of 2008. million of in the fourth quarter 2007. the first quarter of 2009. per share for the quarter of the year. second quarter 2009 results are currently reported. was 346. 4 million. from \$ 15. 7 million in the second quarter of 2009 and the first quarter of 2008. for the second quarter ended june,. quarter as compared to the same period in 2008. million from \$ 46. 6 million for the same period in 2007. total revenue of \$ 50. 9 million, an increase of 3. 1 \% compared. \% compared to the company ' s for the 2008 third quarter.
\end{tcolorbox}

\begin{tcolorbox}[colback=cics_bglight_blue]
\textbf{Dataset: LM1B; NLL = 2.58; Entropy = 3.73}\\
\ttfamily
year - to - quarter operating expenses in the hyatt quarter increased by 1. 6 percent in the first three months of 2008, and for the third quarter 2009 of \$ 7. 7 million, this was a 2. 7 \% increase from the third quarter of 2008 due to net operating expenses of \$ 22. 7 million, an increase of \$ 38. 9 million due to higher operating costs and other related to general and administrative costs. reversing an rise in net income in the fiscal year. in the same period 2008, total revenues of \$ 12 1 million for the second quarter was increased by 12. 8 \%.
\end{tcolorbox}

\begin{tcolorbox}[colback=cics_bglight_blue]
\textbf{Dataset: LM1B; NLL = 3.80; Entropy = 3.94}\\
\ttfamily
" if they want to do that, since the end of 2008, they will know that the global economic downturn will be running at between 4 \% and only 4 \% this year, and that the world economy will be struggling, " said fernando santiago korobattimimo, the head of argentina ' s national environment and environment administration, at the official press conference in brazil, one of the largest developed economies of these two countries.
\end{tcolorbox}

\begin{tcolorbox}[colback=cics_bglight_blue]
\textbf{Dataset: LM1B; NLL = 4.06; Entropy = 4.18}\\
\ttfamily
the thing on the west end of this debate, especially, given all the important events that has already happened in south carolina is that we don ' t vote on the president ' s own performance at the democratic levels as well ) - - and, compared with a lot of voters and sales voters in arizona in the past elections, we know that, as the report showed, we aren ' t going to do the same for all the delegates in and florida because they won in massachusetts, they looked on to change america and weather the storm even before obama was about to change his strategy and win in iowa and lead in the general election.
\end{tcolorbox}

\begin{tcolorbox}[colback=cics_bglight_blue]
\textbf{Dataset: LM1B; NLL = 4.52; Entropy = 4.11}\\
\ttfamily
the move to create this fuel cell product, because of the fast fuel supply chain that has cut gas costs, improved fuel efficiency and enough fueling capacity to replace the smart car vehicles, will have touched off clear assurances made to the local gas industry that if it was no longer made more efficient by the threat of economic and climate changes in japan and other large developed nations, they would be seen as a high earner because of a drop in fuel - like demand.
\end{tcolorbox}

\begin{tcolorbox}[colback=cics_bglight_blue]
\textbf{Dataset: LM1B; NLL = 4.96; Entropy = 4.03}\\
\ttfamily
and in their time, with so much talk of an interest rate cut, the concern that fed ' s aggressive rate cuts, including a dramatic cut in interest rates, and a slow recovery for an already troubled economy, will bring many economic analysts to a close eye on them in a conference call friday - - hours after news earlier this week - - that the fed could keep raising interest rates and cut borrowing and cut spending through this burden at the same time as holding a " smart line, " to start with the economy, the economy, the economy and the economy, which will each increase demand for the economy for years to come.
\end{tcolorbox}